\documentclass{article}

\PassOptionsToPackage{numbers, compress}{natbib}

\usepackage[preprint]{neurips_2026}

\usepackage{microtype}
\usepackage{graphicx}
\usepackage{subcaption}
\usepackage{booktabs} 
\usepackage{multirow}

\usepackage{amsmath}
\usepackage{amssymb}
\usepackage{mathtools}
\usepackage{amsthm}
\usepackage{seqsplit}

\newcommand{\PAIRFormer}{\textsc{PAIR-Former}}
\newcommand{\BRMIL}{\textsc{BR-MIL}}

\newcommand{\CTS}{\textsc{CTS}}

\newcommand{\threeprimeUTR}{\ensuremath{3^{\prime}\mathrm{UTR}}}

\newcommand{\CTSs}{\CTS{}s}

\usepackage{algorithm}
\usepackage{algorithmic}  

\usepackage{xcolor}

\newcommand{\code}[1]{\texttt{#1}}          

\theoremstyle{plain}
\newtheorem{theorem}{Theorem}[section]
\newtheorem{proposition}[theorem]{Proposition}
\newtheorem{lemma}[theorem]{Lemma}
\newtheorem{corollary}[theorem]{Corollary}
\usepackage{makecell}
\theoremstyle{definition}
\newtheorem{definition}[theorem]{Definition}
\newtheorem{assumption}[theorem]{Assumption}
\theoremstyle{remark}
\newtheorem{remark}[theorem]{Remark}

\usepackage[disable,textsize=tiny]{todonotes}

\usepackage{hyperref}
\usepackage[capitalize,noabbrev]{cleveref}

\title{PAIR-Former: Budgeted Relational Multi-Instance Learning for Functional miRNA Target Prediction}

\author{%
  Jiaqi Yin \\
  School of Future Technology \\
  Harbin Institute of Technology \\
  Harbin, China \\
  \texttt{yjqhit@gmail.com} \\
  \And
  Baiming Chen \\
  School of Medicine \\
  Chinese University of Hong Kong, Shenzhen \\
  Shenzhen, China \\
  \AND
  Jia Fei \\
  Department of Biochemistry and Molecular Biology \\
  Medical College, Jinan University \\
  Guangzhou, China \\
  \And
  Mingjun Yang\thanks{Corresponding author.} \\
  Shenzhen Jingtai Technology Co., Ltd. (XtalPi) \\
  Shenzhen, China \\
  \texttt{mingjun.yang@xtalpi.com}
}

\begin{document}

\maketitle

\begin{abstract}
Functional miRNA--mRNA targeting is a large-bag prediction problem where each transcript yields a heavy-tailed pool of candidate target sites (CTSs), yet only a pair-level label is observed.
Prior methods use max-pooling over individual CTS scores, ignoring relational patterns among sites, but modeling these patterns is critical for accuracy.
The challenge is that naive relational aggregation incurs $\mathcal{O}(n^2)$ cost, prohibitive when $n$ reaches thousands, yet a cheap scan alone discards the very interactions that drive functional repression.
We formalize this tension as \emph{Budgeted Relational Multi-Instance Learning (BR-MIL)}, a new MIL problem where the compute budget $K$ is a first-class constraint such that at most $K$ instances per bag may receive expensive encoding and relational processing.
We establish theoretical foundations for BR-MIL, proving that both approximation quality and generalization are governed by $K$ rather than the raw bag size $n$.
Building on this theory, we propose \textbf{PAIR-Former}, which scans all candidates cheaply, selects $K$ diverse CTSs, and aggregates them via Set Transformer.
PAIR-Former achieves state-of-the-art performance, outperforming all reproduced baselines with F1$=0.840$ on miRAW (10-fold balanced CV) and $0.839$ on deepTargetPro in transfer evaluation, while achieving $0.793$ on the large-scale MTI benchmark (420K pairs, $38\times$ larger), demonstrating that budgeted relational MIL scales where naive approaches fail.
Additional results on CAMELYON16 and Musk2 further show that the proposed BR-MIL formulation extends beyond biological sequence modeling.
\end{abstract}

\paragraph{Code availability.} An anonymized implementation is included in the supplementary material; we will release the full codebase upon acceptance.

\begin{figure*}[!t]
  \centering
  \includegraphics[width=0.9\textwidth]{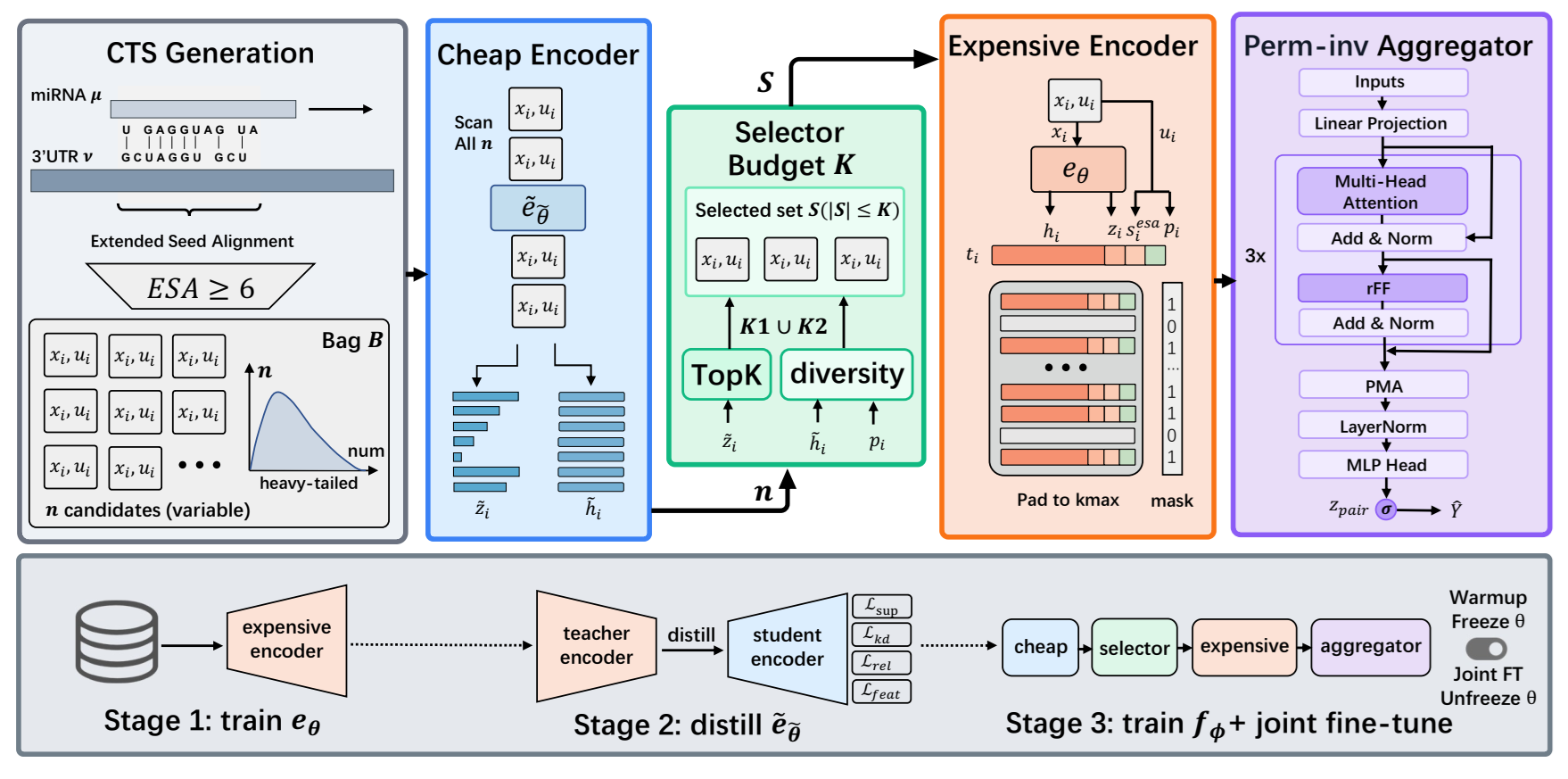}
    \caption{
    \textbf{Overview of \PAIRFormer{} under \BRMIL{}.}
    \textbf{Top:} Inference follows a scan--select--aggregate pipeline: all \CTSs{} are scanned cheaply, a budgeted subset $S$ with $|S|\le K$ is selected, and only selected candidates receive expensive encoding and Set Transformer aggregation for pair-level prediction.
    \textbf{Bottom:} Training consists of \CTS{} encoder pretraining, cheap encoder distillation, and pair-level relational aggregation under the same budgeted forward pass.
    }
  \label{fig:overview}
\end{figure*}


\section{Introduction}

MicroRNAs (miRNAs) are key post-transcriptional regulators that modulate gene expression by binding to messenger RNAs (mRNAs), with broad roles in development, immunity, and disease.
Predicting functional miRNA--mRNA interactions is important for understanding gene regulatory networks and identifying therapeutic targets, but supervision is typically available only at the miRNA--transcript pair level.
Meanwhile, each pair yields a heavy-tailed pool of candidate target sites (\CTSs) within the transcript's \threeprimeUTR{}, creating a large-bag prediction problem in which the observed label aggregates latent contributions from multiple sites without instance-level supervision.

Most modern miRNA target prediction pipelines follow a site-first decomposition: they generate candidate target sites (\CTSs), score each miRNA--\CTS{} window independently, and aggregate site scores into a transcript-level prediction, often by max pooling \cite{targetnet,miraw,mitds,tecmitarget}.
Although window-level encoders have improved substantially, max pooling imposes a \emph{strongest-site assumption}: the pair label is driven by the single most confident site, while other \CTSs{} are treated as conditionally independent given the miRNA.
This bias is restrictive because transcript-level repression may reflect cross-site evidence, including redundancy, cooperative effects, or competitive/decoy-like binding \cite{bartel2009,saetrom2007,ceRNAreview}.
Replacing max pooling with relational aggregation, such as self-attention over all candidate sites, is natural but computationally prohibitive for heavy-tailed bags, incurring $\mathcal{O}(n^2)$ time and memory when a transcript yields $n$ candidates.
As shown in Fig.~\ref{fig:n_distribution}, this regime is typical across benchmarks: median $n=640$--$993$, and more than 92\% of pairs exceed the default budget $K=64$.

We address this challenge by introducing \emph{Budgeted Relational Multi-Instance Learning (BR-MIL)}, a compute-aware MIL formulation where all candidates may be scanned cheaply, but at most $K$ instances per bag receive expensive encoding and relational aggregation.
This reduces expensive encoding from $\mathcal{O}(n)$ to $\mathcal{O}(K)$ and relational aggregation from $\mathcal{O}(n^2)$ to $\mathcal{O}(K^2)$ while preserving full-pool candidate discovery.
We instantiate \BRMIL{} with \textbf{PAIR-Former} (Pool-Aware Instance-Relational Transformer), which cheaply scans all \CTSs{}, selects a diverse budgeted subset, and applies an expensive encoder followed by a permutation-invariant Set Transformer.
We further provide theory showing how $K$ controls the approximation--generalization tradeoff.

We evaluate \PAIRFormer{} on three miRNA targeting benchmarks and two cross-domain MIL tasks.
Under 10-fold balanced cross-validation, it achieves F1$=0.840$ on miRAW and $0.839$ on deepTargetPro, and $0.793$ on the 420K-pair MTI benchmark, where it operates over nearly half a billion candidate \CTS{} windows.
It also improves deepTargetPro transfer performance from the prior best F1 $0.791$ to $0.839$.
Results on CAMELYON16 and Musk2, together with controlled budget and runtime analyses, further support the generality of \BRMIL{} and the predicted accuracy--compute tradeoff.

\begin{figure*}[t]
  \centering
  \includegraphics[width=0.9\textwidth]{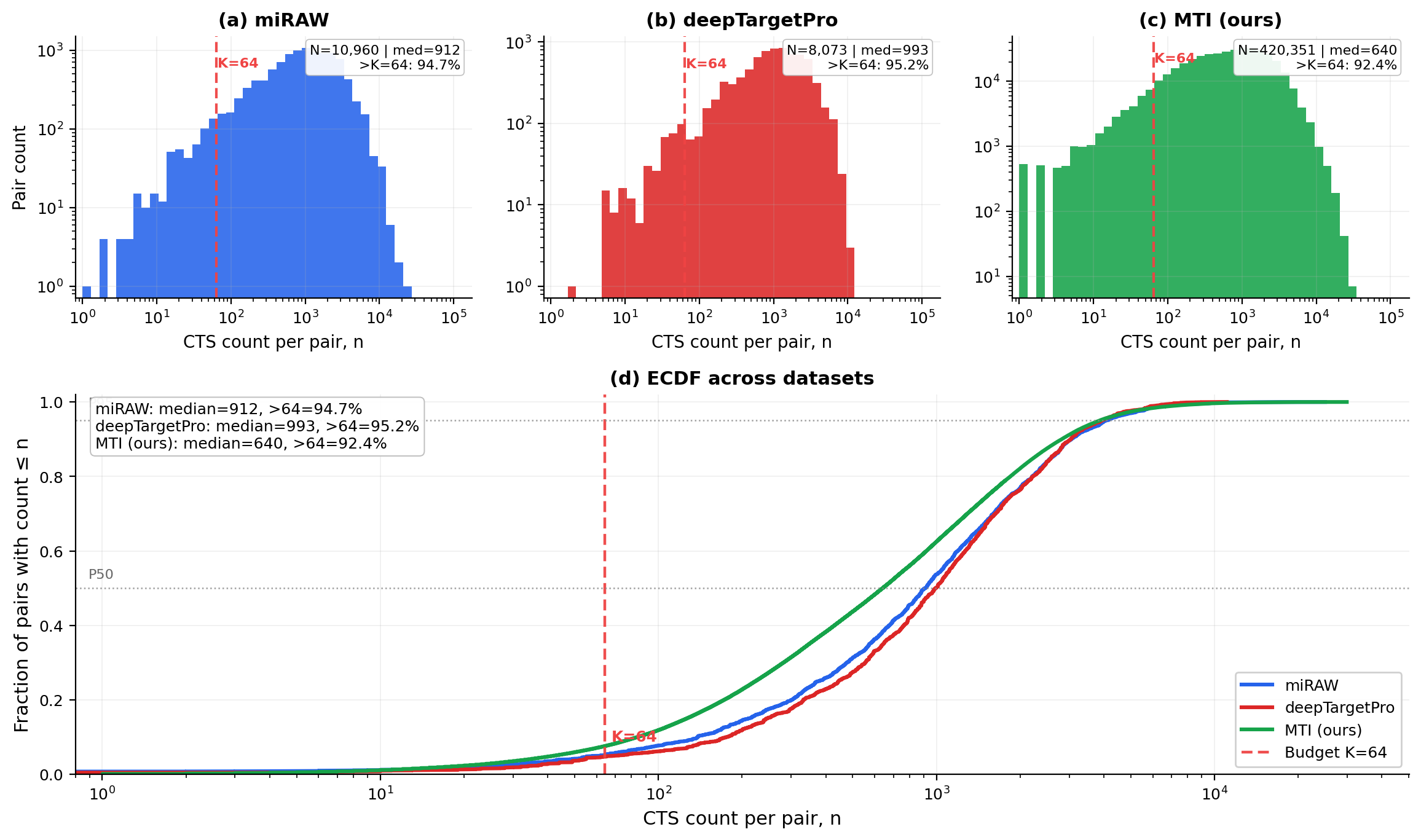}
    \caption{
    \textbf{Candidate pool size distributions.}
    Top: log-scale histograms of per-pair \CTS{} counts after ESA filtering.
    Bottom: empirical CDFs.
    Across miRAW, deepTargetPro, and MTI, median $n$ ranges from 640 to 993, far above $K=64$ (red dashed line), and 92--95\% of pairs exceed this budget.
    Full statistics are in Appendix~\ref{app:n_distribution_stats}.
    }
  \label{fig:n_distribution}
\end{figure*}

We summarize our contributions as follows:

\noindent -- We formulate \textbf{Budgeted Relational MIL}, where a hard budget $K$ limits expensive encoding and relational aggregation in large, heavy-tailed bags, and provide theory linking $K$ to approximation and generalization.

\noindent -- We develop \textbf{\PAIRFormer{}}, a scan--select--aggregate framework that cheaply scans all \CTSs{}, selects a diverse subset, and applies Set Transformer aggregation.

\noindent -- We validate \PAIRFormer{} on miRAW, deepTargetPro, and the 420K-pair MTI benchmark, with cross-domain results and controlled accuracy--compute analyses.


\section{Related Work}
\label{sec:related_work}

\paragraph{Functional miRNA target prediction.}
Classical miRNA target predictors use hand-crafted signals such as seed complementarity, conservation, accessibility, and thermodynamic stability, as in TargetScan, miRanda, PITA, mirSVR, and miRDB \cite{targetscan,miranda,pita,mirsvr,mirdb}.
Modern deep models reduce feature engineering by learning sequence- or alignment-driven representations \cite{deeptarget,miraw,mimosa,deepmirTar,mirformer}, graph-structured embeddings \cite{graphtar}, or image-based representations \cite{micgr}.
A common site-first pipeline enumerates candidate target sites (\CTSs), scores each miRNA--\CTS{} window independently, and aggregates site scores to the transcript level, often by max pooling \cite{targetnet,miraw,mitds,tecmitarget,miTAR}.
Despite stronger window-level encoders, transcript-level inference still largely follows a strongest-site assumption and treats sites as independent given the miRNA \cite{targetnet,saetrom2007}.

\paragraph{MIL, set aggregation, and scalability.}
Functional targeting is naturally related to multi-instance learning (MIL), where only bag-level labels are observed and instance labels are latent \cite{dietterich1997mil,carbonneau2018mil}.
Permutation-invariant set models provide a principled interface for MIL: DeepSets uses simple pooling \cite{deepsets}, attention-based MIL learns instance weights \cite{abmil}, and Set Transformer models higher-order interactions through self-attention \cite{settransformer}.
Large-bag MIL has also been studied in pathology, where whole-slide images contain thousands of patches and scalable methods use selection, clustering, pseudo-bags, zooming, or transformers \cite{campanella2019clinical,lu2021clam,thandiackal2022zoommil,zhang2022dtfd,transmil}.
Unlike pathology, miRNA targeting yields heavy-tailed, sequence-derived \CTS{} pools without a natural 2D spatial hierarchy; \BRMIL{} therefore makes the budget $K$ part of the problem formulation and instantiates it through cheap scanning, budgeted selection, and permutation-invariant relational aggregation.

\section{Problem Formulation}
\label{sec:problem-formulation}

\paragraph{Large-bag functional targeting.}
Each miRNA--mRNA pair defines a bag of candidate target sites (\CTSs)
$B=\{(x_i,u_i)\}_{i=1}^{n}$, where $x_i$ denotes the \CTS{} content
(e.g., an alignment-aware window tensor) and $u_i$ contains structural attributes
(e.g., normalized transcript position and ESA score).
The bag size $n$ is large and heavy-tailed across transcripts (Fig.~\ref{fig:n_distribution}), while supervision is available only at the pair level through a binary label $Y\in\{0,1\}$.
The goal is to learn a permutation-invariant predictor $\hat Y(B)$ for transcript-level functional repression.

\paragraph{Budgeted Relational MIL.}
We define \BRMIL{} as multi-instance learning with large, heavy-tailed bags under a hard per-bag compute budget:
expensive instance encoding and relational aggregation may be applied to at most $K$ instances from each bag.
This budgeted setting separates cheap full-pool scanning from expensive relational reasoning.

\paragraph{Budgeted predictor and cost.}
A cheap encoder $\tilde e_{\tilde\theta}$ scans all $n$ candidates to produce lightweight signals
(e.g., cheap embeddings $\tilde h_i$ and logits $\tilde z_i$).
A selector $\pi_\psi$ then returns a subset $S(B)\subseteq[n]$ with $|S(B)|\le K$.
Only selected instances are processed by the expensive encoder and aggregated by a permutation-invariant set function:
\begin{equation}
\begin{split}
\hat Y(B)
&= f_\phi\Big(\big\{\tau(e_\theta(x_i,u_i),u_i)\big\}_{i\in S(B)}\Big),
\\
S(B)
&=\pi_\psi\Big(\big\{\tilde e_{\tilde\theta}(x_i,u_i)\big\}_{i=1}^{n}\Big),
\end{split}
\label{eq:brmil_core}
\end{equation}
where $e_\theta$ is the expensive encoder, $\tau$ is a tokenization function, and
$f_\phi$ is a permutation-invariant aggregator such as a Set Transformer.
Permutation invariance requires $\hat Y(B)=\hat Y(\sigma(B))$ for any permutation $\sigma$ of the instances,
which holds when selection depends only on the multiset of cheap signals, up to deterministic tie-breaking, and $f_\phi$ is permutation invariant.

BR-MIL addresses two bottlenecks: expensive per-instance encoding, which would require $n$ expensive forward passes, and relational aggregation, which costs $\mathcal{O}(n^2)$ for full self-attention.
Restricting expensive processing to $K$ selected instances reduces these costs to $\mathcal{O}(K)$ and $\mathcal{O}(K^2)$, while retaining a cheap $\mathcal{O}(n)$ full-pool scan.

Full notation, objective details, and additional invariance conditions are deferred to Appendix~\ref{app:problem_setup_long}.

\section{Method}
\label{sec:method}

\subsection{Overview and candidate representation}
Given a miRNA--mRNA pair with a bag $B=\{(x_i,u_i)\}_{i=1}^{n}$ of candidate target sites (\CTSs), \PAIRFormer{} instantiates \BRMIL{} as a scan--select--aggregate pipeline.
A cheap encoder first scans all $n$ candidates, a budgeted selector chooses a subset $S$ with $|S|\le K$, and an expensive encoder followed by a Set Transformer performs relational prediction only on the selected candidates.
Figure~\ref{fig:overview} illustrates the complete pipeline.

We extract \CTSs{} from the \threeprimeUTR{} using a 40-nt sliding window with stride 1 and retain candidates passing ESA filtering ($s_i^{\mathrm{esa}}\ge 6$).
For each retained candidate, we construct an alignment-aware input tensor $X_i\in\mathbb{R}^{10\times 50}$ following the TargetNet protocol \cite{targetnet,mitds}.
Each candidate is augmented with structural attributes $u_i=(p_i,s_i^{\mathrm{esa}})$, where $p_i$ is the normalized transcript position and $s_i^{\mathrm{esa}}$ is the ESA score.

\subsection{Budgeted scan--select--encode}
\label{sec:stselector}
\PAIRFormer{} uses two \CTS{} encoders with different computational costs.
The expensive encoder $e_{\theta}$ is a TargetNet-style encoder with channel attention and outputs
\begin{equation}
(h_i,z_i)=e_{\theta}(x_i,u_i), \qquad h_i\in\mathbb{R}^{384},~z_i\in\mathbb{R}.
\end{equation}
The cheap encoder $\tilde e_{\tilde\theta}$ outputs
\begin{equation}
(\tilde h_i,\tilde z_i)=\tilde e_{\tilde\theta}(x_i,u_i), \qquad \tilde h_i\in\mathbb{R}^{64},~\tilde z_i\in\mathbb{R},
\end{equation}
and is trained by distillation from the expensive encoder.
At inference and pair-level training time, $\tilde e_{\tilde\theta}$ scans all $n$ candidates, while $e_\theta$ is evaluated only on the selected subset.

For each bag, we set $K=\min(\texttt{kmax},n)$ and split the budget into $K_1=\lfloor \rho K\rfloor$ and $K_2=K-K_1$.
STSelector first selects the top-$K_1$ candidates by cheap logit $\tilde z_i$, then allocates the remaining $K_2$ slots across transcript position bins while removing near-duplicates via SimHash on cheap embeddings $\tilde h_i$.
This encourages coverage across the transcript while prioritizing high-confidence candidates.
The selector runs in near-linear time on CPU; full pseudocode and hyperparameters are provided in Appendix~\ref{app:selector}.

\subsection{Relational aggregation}
For each selected \CTS{} $i\in S$, we construct a token by concatenating the expensive embedding, expensive logit, and structural attributes:
\begin{equation}
t_i = [h_i\,\Vert\, z_i\,\Vert\, s_i^{\mathrm{esa}}\,\Vert\, p_i]\in\mathbb{R}^{387}.
\label{eq:token}
\end{equation}
The selected token set is padded to length \texttt{kmax} with a binary mask.
A Set Transformer with self-attention blocks aggregates the masked token set into a pair-level logit $z_{\mathrm{pair}}$, yielding the prediction $\hat Y=\sigma(z_{\mathrm{pair}})$.
Because aggregation is performed over a set and padding is masked, the prediction is invariant to the ordering of selected instances.

\subsection{Training and inference}
\label{sec:training}
We train \PAIRFormer{} in three stages.
First, the expensive encoder $e_\theta$ is trained on \CTS{}-level data with binary supervision.
Second, the cheap encoder $\tilde e_{\tilde\theta}$ is distilled from $e_\theta$ using a combination of supervised loss, logit distillation, and feature matching:
\begin{equation}
L_{\mathrm{distill}}
= (1-\alpha)\,L_{\mathrm{sup}} + \alpha\,L_{\mathrm{KD}} + \beta_{\mathrm{feat}}\,L_{\mathrm{feat}}.
\label{eq:distill}
\end{equation}
Third, the Set Transformer aggregator is trained on pair-level data using the full budgeted forward pass; we first freeze the instance encoders and train the aggregator, then jointly fine-tune the expensive encoder and aggregator.
Binary cross-entropy with logits is used for both \CTS{}-level and pair-level classification.
Full loss definitions and training hyperparameters are provided in Appendix~\ref{app:loss}.

At test time, \PAIRFormer{} follows the same budgeted forward pass: generate \CTSs{}, scan all candidates with the cheap encoder, select $S$ with $|S|=K$, encode only selected candidates with $e_\theta$, tokenize them via Eq.~\eqref{eq:token}, and aggregate the masked token set with the Set Transformer.
Pseudocode for training and inference is provided in Appendix~\ref{app:algorithms}.

\section{Theoretical Analysis}
\label{sec:theory}

We provide two results that justify the \BRMIL{} design.
First, budgeted relational prediction approximates full-pool relational prediction through the coverage of influential instances.
Second, for a fixed selector, the dominant generalization complexity of the expensive relational component depends on the budget $K$ rather than the raw bag size $n$.

\begin{theorem}[Approximation under budgeted selection]
\label{thm:approx}
Let $\hat Y_{\mathrm{full}}(B)$ denote the prediction obtained by applying the relational aggregator to all expensive tokens in a bag, and let $\hat Y_K(B)$ denote the budgeted prediction obtained by masking all tokens outside a selected subset $S(B)$ with $|S(B)|\le K$.
Under Assumptions~\ref{assm:A0}--\ref{assm:A2} in Appendix~\ref{app:theory}, for any fixed bag $B$,
\begin{equation}
\mathbb{E}\!\left[
\left|\hat Y_{\mathrm{full}}(B)-\hat Y_K(B)\right|
\,\middle|\,B
\right]
\le
2RL_\star\left(\psi_K(B)+\Delta_K^{\mathrm{w}}(B)\right),
\label{eq:main_approx_bound}
\end{equation}
where $L_\star$ is the masking-sensitivity envelope, $\psi_K(B)$ is the oracle top-$K$ influence tail, and $\Delta_K^{\mathrm{w}}(B)$ is the selector's influence regret.
\end{theorem}

This bound shows that approximation improves as $K$ reduces the oracle tail or as the selector better tracks influential instances.

\begin{theorem}[Generalization controlled by budget $K$]
\label{thm:gen}
Condition on a selector and selected-token representation fixed independently of the training labels, and let $\mathcal{F}_K$ be a capacity-controlled class of permutation-invariant relational aggregators that receive at most $K$ selected tokens per bag.
For $M$ training bags, with probability at least $1-\delta$, every $f\in\mathcal{F}_K$ satisfies
\begin{equation}
\mathcal{R}(f)
\le
\widehat{\mathcal{R}}_M(f)
+
\tilde{\mathcal{O}}\!\left(\sqrt{\frac{K}{M}}\right)
+
\mathcal{O}\!\left(\sqrt{\frac{\log(1/\delta)}{M}}\right).
\label{eq:main_gen_bound}
\end{equation}
If the selector is chosen from a finite class $\mathcal{S}_K$, an additional selector-complexity term of order $\sqrt{\log |\mathcal{S}_K|/M}$ appears.
\end{theorem}

This component-level result separates candidate discovery from expensive relational aggregation: after selection, the statistical complexity of the relational component scales with $K$, not directly with $n$; any $n$ dependence enters through the selector class.

Together, Theorems~\ref{thm:approx} and~\ref{thm:gen} characterize the budget tradeoff in \BRMIL{}: increasing $K$ can reduce oracle tail mass and selector regret, while also increasing the complexity of the expensive relational component.
Full assumptions, proofs, and practical guidance for choosing $K$ are provided in Appendix~\ref{app:theory}.

\section{Experiments}
\label{sec:experiments}

\subsection{Experimental setup}
\label{subsec:exp_setup}

\paragraph{Datasets.}
We evaluate functional miRNA--mRNA prediction on three benchmarks: (i) miRAWtest (10 predefined subsets), (ii) deepTargetPro (8{,}073 pairs, 10-fold CV), and (iii) MTI (420{,}351 pairs, large-scale CLASH/chiRA/HYBRID data).

For miRAWtest, we use 10-fold balanced cross-validation: each predefined subset serves once as the held-out fold, with 109 positive and 109 negative test pairs per fold.
The same protocol is applied to deepTargetPro, and all reproduced methods and our models use identical test folds.

Stage 1--2 training uses TargetNet's \CTS{}-level data; Stage 3 uses pair-level labels.
We report PR-AUC as the primary metric and F1@0.5 as a complementary thresholded metric.

\paragraph{Baselines.}
We reproduce TargetNet \cite{targetnet} using its official pre-trained checkpoint and codebase, and Mimosa \cite{mimosa} using its official pre-trained model, both evaluated under our 10-fold protocol.
We also include a max-pooling baseline built from our optimized CTS encoder.
Our main method is \textbf{\PAIRFormer{}} with \BRMIL{} and default budget $K{=}64$.
Quoted baselines (miTDS, PITA, etc.) are shown for context but are not directly comparable due to different evaluation protocols.

\paragraph{Implementation.}
We use $K{=}64$ as the default budget, evaluate miRAW and deepTargetPro under 10-fold balanced cross-validation, and select checkpoints by validation PR-AUC.
Small-scale experiments (miRAW, deepTargetPro) run on a single RTX 4090 GPU; large-scale MTI training uses 8$\times$A100 GPUs.

\subsection{Main miRNA targeting results}
\label{subsec:main_results}

\paragraph{Main result.}

Under 10-fold balanced cross-validation with budget $K^\star{=}64$, \PAIRFormer{} achieves PR-AUC $0.869{\pm}0.031$ and F1@0.5 $0.840{\pm}0.022$, outperforming reproduced TargetNet, Mimosa, and max-pooling baselines (F1 $0.779$--$0.798$).
The improvement is consistent across all 10 folds.

\begin{table*}[!t]
  \centering
    \caption{
    \textbf{miRAW 10-fold balanced cross-validation results.}
    Mean$\pm$std over 10 test folds (218 pairs each: 109 positive + 109 negative).
    All reproduced methods use the official pre-trained checkpoints evaluated under our identical 10-fold protocol.
    Quoted prior results use their original full-miRAW protocols and are included only as context, not for direct comparison.
    $^\dagger$Results quoted from \cite{micgr}, which re-evaluates both methods under its own pairing-pattern CTS protocol (${\sim}8$ avg.\ CTSs/gene) rather than ESA filtering.
    F1@0.5 uses a fixed threshold of 0.5.
    }

  \label{tab:overall}

  \small
  \renewcommand{\arraystretch}{1.08}

  \resizebox{\linewidth}{!}{\begin{tabular}{lccccccc}
    \toprule
    \textbf{Method} & \textbf{PR-AUC} & \textbf{F1@0.5} & \textbf{Acc} & \textbf{Prec} & \textbf{Rec} & \textbf{Spec} & \textbf{NPV} \\
    \midrule
    \multicolumn{8}{l}{\textit{Reproduced baselines (10-fold balanced CV)}} \\
    TargetNet (official) \cite{targetnet} & 0.773$\pm$0.026 & 0.779$\pm$0.018 & 0.730$\pm$0.021 & 0.660$\pm$0.017 & 0.951$\pm$0.027 & 0.510$\pm$0.031 & 0.913$\pm$0.044 \\
    Mimosa (official) \cite{mimosa} & 0.740$\pm$0.028 & 0.788$\pm$0.017 & 0.742$\pm$0.024 & 0.675$\pm$0.024 & 0.954$\pm$0.027 & 0.527$\pm$0.047 & 0.925$\pm$0.034 \\
    \midrule
    \multicolumn{8}{l}{\textit{Ours (10-fold balanced CV)}} \\
    Max pooling (TN-Opt encoder + max) & 0.808$\pm$0.023 & 0.798$\pm$0.017 & 0.752$\pm$0.024 & 0.673$\pm$0.021 & 0.982$\pm$0.009 & 0.522$\pm$0.041 & 0.965$\pm$0.017 \\
    \textbf{PAIR-Former (ours)} & \textbf{0.869}$\pm$\textbf{0.031} & \textbf{0.840}$\pm$\textbf{0.022} & \textbf{0.831}$\pm$\textbf{0.026} & \textbf{0.790}$\pm$\textbf{0.040} & \textbf{0.896}$\pm$\textbf{0.037} & \textbf{0.768}$\pm$\textbf{0.058} & \textbf{0.892}$\pm$\textbf{0.032} \\
    \midrule
    \multicolumn{8}{l}{\textit{Quoted baselines (full miRAW protocol; not strictly comparable)}} \\
    miTDS (quoted) \cite{mitds}        & \textemdash & 0.8063 & 0.7700 & 0.6962 & 0.9578 & 0.5821 & 0.9326 \\
    PITA (quoted) \cite{pita}          & \textemdash & 0.2162 & 0.5053 & 0.5196 & 0.1365 & 0.8741 & 0.5030 \\
    miRDB (quoted) \cite{mirdb}        & \textemdash & 0.2110 & 0.5373 & 0.7135 & 0.1239 & 0.9507 & 0.5205 \\
    miRanda (quoted) \cite{miranda}    & \textemdash & 0.3568 & 0.5001 & 0.4997 & 0.2775 & 0.7226 & 0.5001 \\
    TargetScan (quoted) \cite{targetscan} & \textemdash & 0.4712 & 0.5577 & 0.5852 & 0.3945 & 0.7208 & 0.5436 \\
    deepTarget (quoted) \cite{deeptarget} & \textemdash & 0.4904 & 0.6521 & 0.8332 & 0.3477 & 0.9354 & 0.6064 \\
    miRAW (quoted) \cite{miraw}        & \textemdash & 0.7289 & 0.7055 & 0.6749 & 0.7923 & 0.6186 & 0.7493 \\
    miCGR$^\dagger$ (quoted) \cite{micgr} & \textemdash & 0.8009 & 0.7902 & 0.7590 & 0.8479 & 0.7331 & 0.8296 \\
    TEC-miTarget$^\dagger$ (quoted) \cite{tecmitarget} & \textemdash & 0.7817 & 0.7796 & 0.7744 & 0.7892 & 0.7701 & 0.7852 \\
    \bottomrule
  \end{tabular}}
\end{table*}

\paragraph{External validation on deepTargetPro.}

To validate generalization beyond miRAWtest, we evaluate on deepTargetPro \cite{lee2020deeptargetpro} under the same 10-fold balanced protocol (Table~\ref{tab:deeptargetpro}).
With miRAW-pretrained CTS encoders and only the Stage~3 aggregator trained on deepTargetPro, \PAIRFormer{} achieves F1$=83.9{\pm}3.9\%$, outperforming TEC-miTarget ($79.11\%$) by $+4.8$ points.
Training all three stages from scratch yields F1$=83.2{\pm}3.2\%$.
Because the CTS encoders in the transfer setting were never trained on deepTargetPro, the gain isolates the contribution of Set Transformer aggregation.

\begin{table*}[t]
  \centering
    \caption{
    \textbf{deepTargetPro external validation.}
    \PAIRFormer{} is evaluated under 10-fold balanced cross-validation.
    The transfer variant reuses miRAW-pretrained \CTS{} encoders and trains only the pair-level aggregator; the full variant trains all three stages on deepTargetPro.
    Baselines are quoted from TEC-miTarget~\cite{tecmitarget} under its published protocol and are included for context.
    All metrics are percentages.
    }
  \label{tab:deeptargetpro}

  \scriptsize
  \setlength{\tabcolsep}{3.2pt}
  \renewcommand{\arraystretch}{1.05}

  \resizebox{0.88\linewidth}{!}{
  \begin{tabular}{lcccccc}
    \toprule
    \textbf{Method} & \textbf{Acc (\%)} & \textbf{Sens (\%)} & \textbf{Spec (\%)} & \textbf{PPV (\%)} & \textbf{NPV (\%)} & \textbf{F1 (\%)} \\
    \midrule
    \multicolumn{7}{l}{\textit{Seed-match-based methods}} \\
    PITA \cite{pita}          & 50.53 & 13.65 & 87.41 & 51.96 & 50.31 & 21.62 \\
    mirSVR \cite{mirsvr}      & 50.01 & 27.76 & 72.26 & 49.97 & 50.01 & 35.68 \\
    miRDB \cite{mirdb}        & 53.73 & 12.39 & 95.07 & 71.35 & 52.05 & 21.10 \\
    microT \cite{microt}      & 61.13 & 58.94 & 63.32 & 61.62 & 60.70 & 60.24 \\
    TargetScan \cite{targetscan} & 55.77 & 39.45 & 72.08 & 58.52 & 54.36 & 47.12 \\
    \midrule
    \multicolumn{7}{l}{\textit{Deep learning methods}} \\
    deepTarget \cite{deeptarget}  & 65.21 & 34.77 & 93.54 & 83.32 & 60.64 & 49.04 \\
    deepTargetPro \cite{lee2020deeptargetpro} & 78.04 & 75.51 & 80.38 & 78.17 & 77.92 & 76.81 \\
    TargetNet \cite{targetnet}    & 72.61 & 95.08 & 51.67 & 64.69 & 91.90 & 76.99 \\
    TEC-miTarget \cite{tecmitarget} & 79.97 & 78.56 & 81.29 & 79.67 & 80.25 & 79.11 \\
    \midrule
    \multicolumn{7}{l}{\textit{Ours: 10-fold balanced CV}} \\
    \textbf{\PAIRFormer{} (transfer)} & \textbf{82.9}$\pm$\textbf{4.9} & \textbf{91.3}$\pm$\textbf{4.9} & \textbf{75.0}$\pm$\textbf{10.5} & \textbf{78.0}$\pm$\textbf{6.4} & \textbf{90.1}$\pm$\textbf{4.7} & \textbf{83.9}$\pm$\textbf{3.9} \\
    \textbf{\PAIRFormer{} (full)} & \textbf{82.7}$\pm$\textbf{3.3} & \textbf{88.2}$\pm$\textbf{5.0} & \textbf{77.6}$\pm$\textbf{6.0} & \textbf{78.9}$\pm$\textbf{4.1} & \textbf{87.1}$\pm$\textbf{4.5} & \textbf{83.2}$\pm$\textbf{3.2} \\
    \bottomrule
  \end{tabular}
  }
\end{table*}

\paragraph{Large-scale validation on MTI.}
To evaluate scalability beyond small class-balanced benchmarks, we further test \PAIRFormer{} on MTI (miRNA Target Interaction), a large-scale benchmark constructed from CLASH \cite{clash}, chiRA \cite{chira}, and HYBRID \cite{hybrid} data.
MTI contains 420,351 miRNA--mRNA pairs, making it $38\times$ larger than miRAWtest, and covers a substantially broader transcript set (44,172 unique mRNAs vs.\ 2,534 in miRAW).
Its \CTS{} distribution is also heavy-tailed, with median $n{=}640$ and 92.4\% of pairs exceeding the default budget $K{=}64$ (Fig.~\ref{fig:n_distribution}; Appendix~\ref{app:n_distribution_stats}).
We train all three stages with the same architecture on 8$\times$A100 GPUs and evaluate budgets $K\in\{64,128,256,512\}$.
As shown in Fig.~\ref{fig:budget_robustness}(a), \PAIRFormer{} achieves its best MTI performance at $K{=}512$, with F1$=0.7925$ and PR-AUC$=0.8729$.
The lower absolute F1 compared with miRAWtest (0.79 vs.\ 0.84) suggests that MTI provides a more challenging large-scale setting.
Performance improves from F1$=0.7708$ at $K{=}64$ to F1$=0.7925$ at $K{=}512$ ($+2.2$pp), indicating that budgeted relational aggregation becomes more valuable as dataset scale and candidate-pool complexity increase.
Training at $K{=}512$ requires approximately 65 minutes per epoch on 8$\times$A100 GPUs.
Relative to the default budget $K{=}64$, processing all candidates would require roughly $10\times$ more expensive encoding at the median and $60\times$ at P95, before considering the quadratic cost of full relational aggregation.

\subsection{Budget, robustness, and runtime analysis}
\label{subsec:budget_analysis}

\paragraph{Budget and pool-size robustness.}
We further analyze the budget behavior and candidate-pool robustness on MTI.
In Fig.~\ref{fig:budget_robustness}(a), we vary $K\in\{1,8,16,32,64,128,256,512\}$.
\textsc{truncate@Kmax} trains once at $K_{\max}=512$ and evaluates smaller budgets by masking, while \textsc{retrain@K} trains a separate budget-matched model for each $K$.
The retrained curve generally improves with larger budgets, consistent with Theorem~\ref{thm:approx}: larger $K$ can reduce uncovered influence mass.
In Fig.~\ref{fig:budget_robustness}(b), we fix the expensive-token budget at $K^\star=64$ and vary the selector-visible pool size $n$.
Performance drops when $n\approx K^\star$ but saturates once $n$ is a few multiples of $K^\star$, suggesting that larger visible pools mainly improve selection quality rather than increasing the capacity of the expensive relational component.

\begin{figure*}[t]
  \centering
  \begin{subfigure}[t]{0.49\textwidth}
    \centering
    \includegraphics[width=\linewidth]{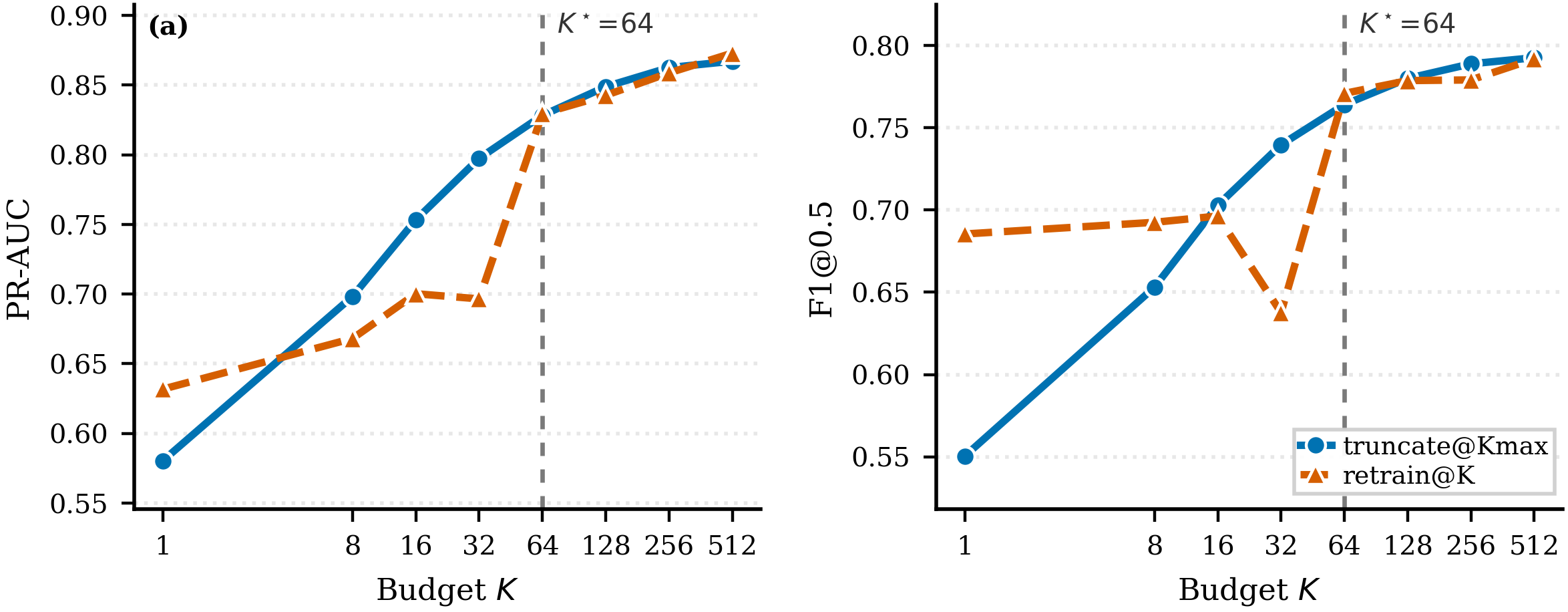}
    \caption{Performance vs.\ expensive-token budget $K$.}
    \label{fig:perf_vs_k}
  \end{subfigure}
  \hfill
  \begin{subfigure}[t]{0.49\textwidth}
    \centering
    \includegraphics[width=\linewidth]{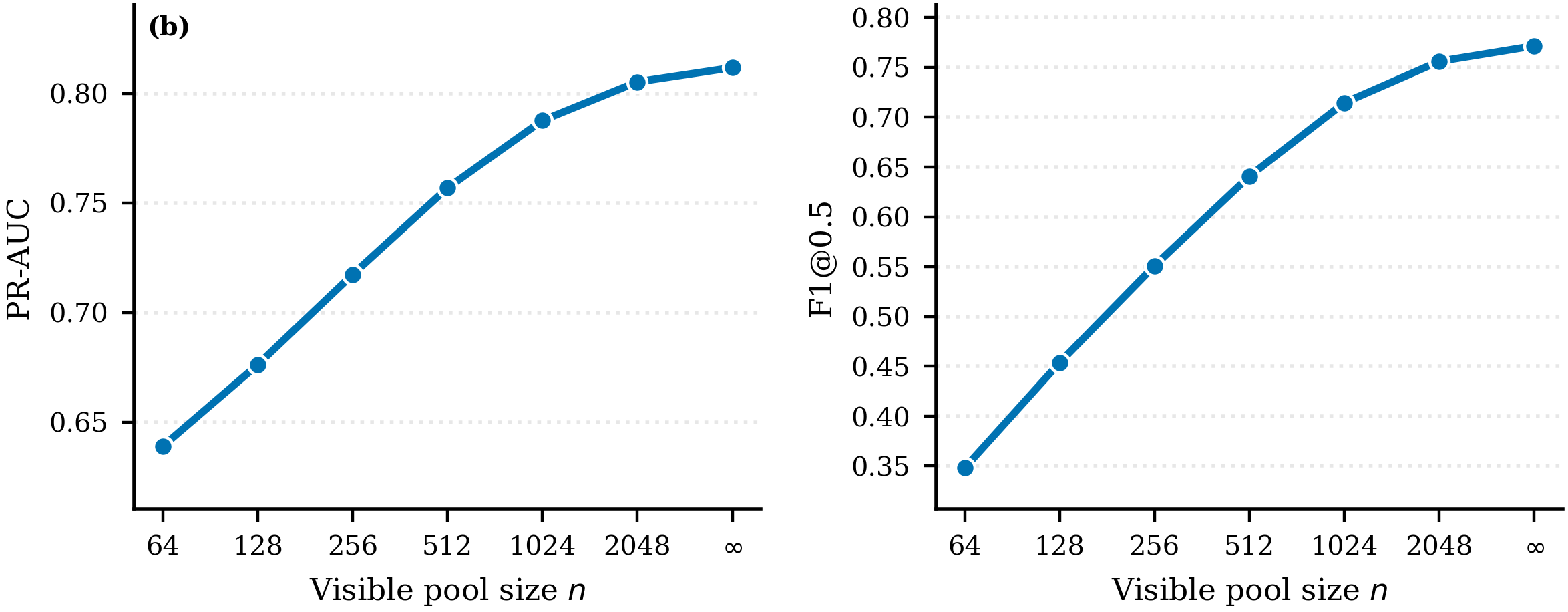}
    \caption{Robustness to visible candidate pool size $n$ at fixed $K^\star=64$.}
    \label{fig:robustness_vs_n}
  \end{subfigure}
\caption{
\textbf{Budget and pool-size analysis on MTI.}
\textbf{(a)} Increasing the expensive-token budget $K$ improves PR-AUC and F1; the retrained model reaches F1$=0.7925$ and PR-AUC$=0.8729$ at $K=512$.
\textbf{(b)} At fixed $K^\star=64$, performance improves as the selector sees more candidates and saturates once the visible pool is a few multiples of $K^\star$.
These results support the predicted budget tradeoff: $K$ controls relational capacity, while larger visible pools improve candidate selection.
}
  \label{fig:budget_robustness}
\end{figure*}

\paragraph{Runtime and memory.}
We profile online inference across budgets, comparing \textsc{BR-MIL\_online} with \textsc{TargetNet\_like\_online} and \textsc{Naive\_online}.
At $K=64$, BR-MIL achieves TargetNet-like latency and throughput while enabling relational aggregation, and peak VRAM remains modest for $K\le64$.
Fig.~\ref{fig:runtime_analysis}b shows that wall-time is dominated by a shared CPU gather step, with only minor BR-MIL-specific overhead; full profiling details are in Appendix~\ref{app:runtime_extra}.

\begin{figure*}[t]
  \centering
  \begin{subfigure}[t]{0.49\textwidth}
    \centering
    \includegraphics[width=\linewidth]{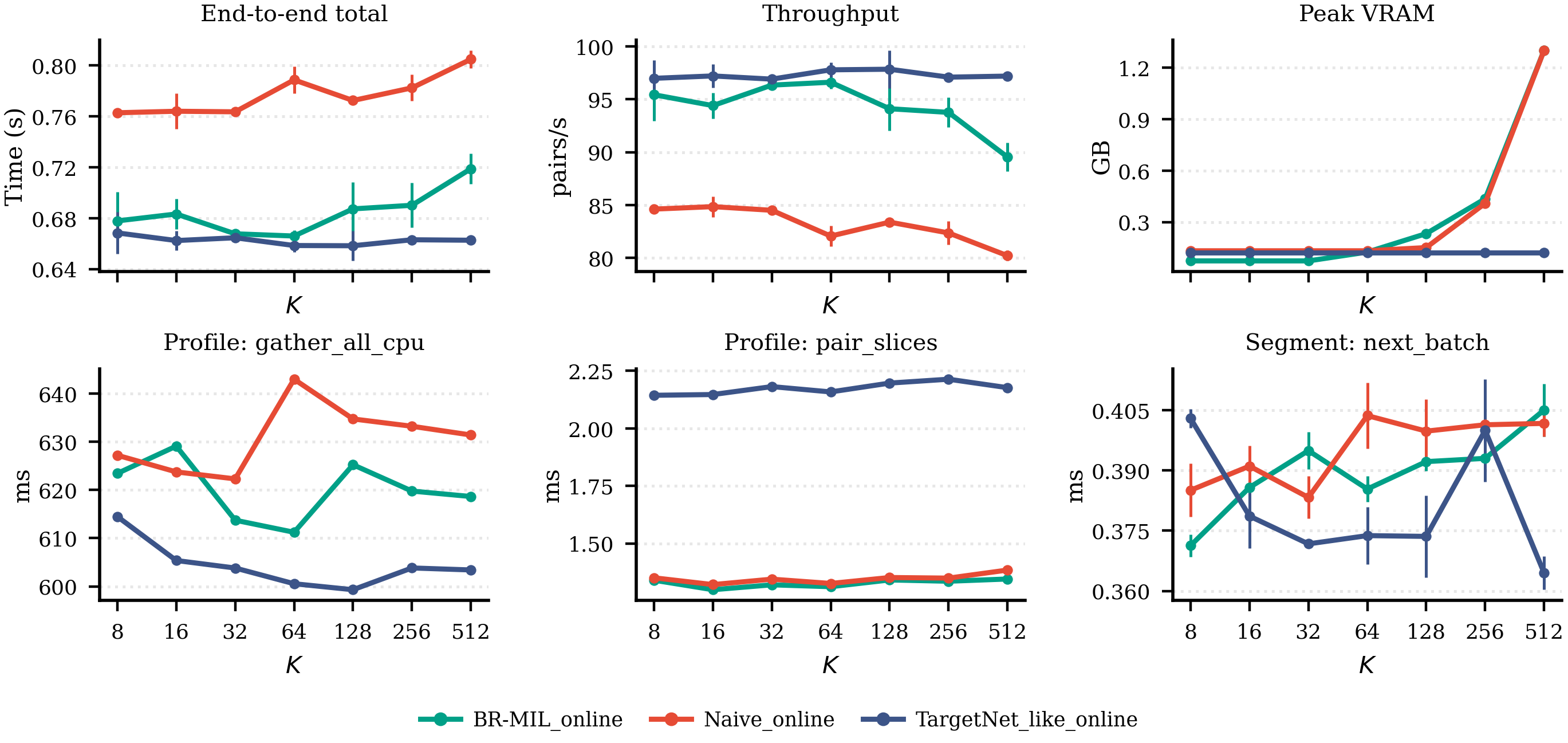}
    \caption{Online inference cost vs.\ budget $K$.}
    \label{fig:compute_vs_k}
  \end{subfigure}
  \hfill
  \begin{subfigure}[t]{0.49\textwidth}
    \centering
    \includegraphics[width=\linewidth]{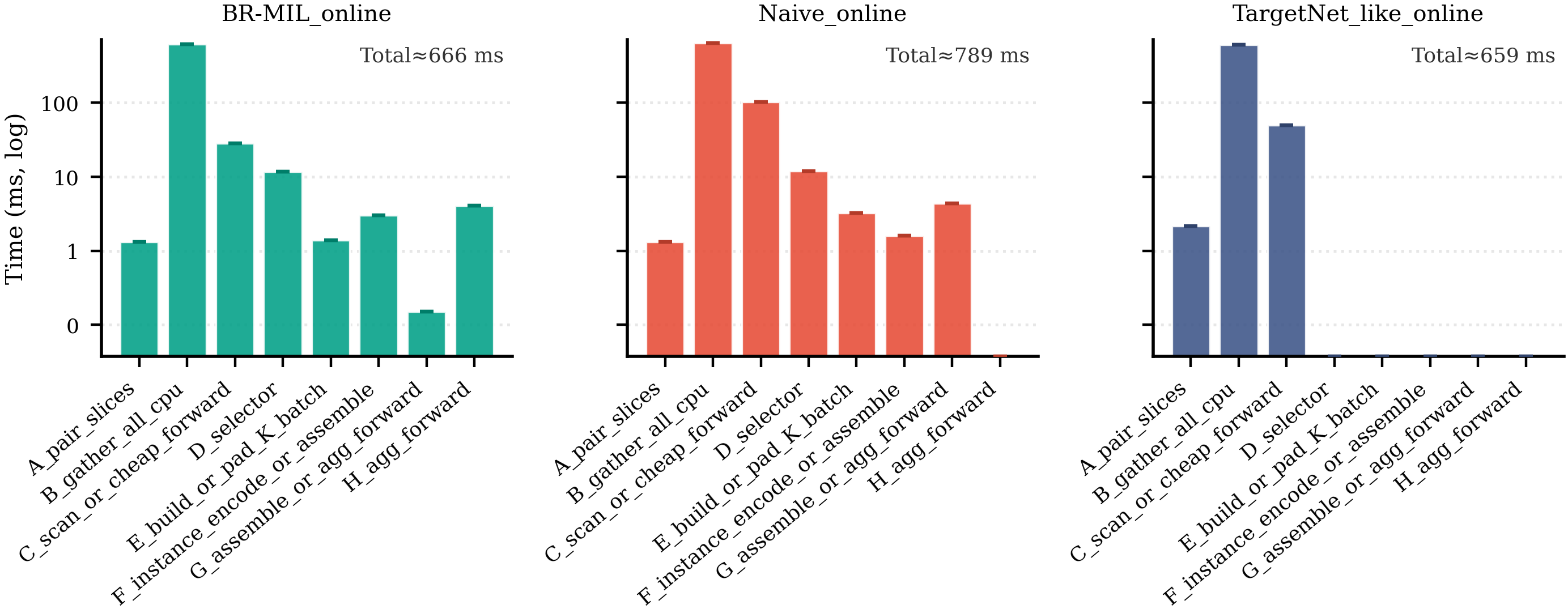}
    \caption{Stage-wise latency breakdown at $K=64$.}
    \label{fig:stage_breakdown}
  \end{subfigure}
  \caption{
  \textbf{Runtime and memory analysis.}
  \textbf{(a)} End-to-end latency, throughput, and peak VRAM across budgets for \textsc{BR-MIL\_online}, \textsc{Naive\_online}, and the budget-independent \textsc{TargetNet\_like\_online} reference.
  \textbf{(b)} Stage-wise latency breakdown at $K=64$ shows that the shared CPU gather stage dominates wall-time, while BR-MIL-specific selection and aggregation introduce only small overhead.
  }
  \label{fig:runtime_analysis}
\end{figure*}

\subsection{Architecture ablations}
\label{subsec:arch_ablation}

Scaling the \CTS{} encoder yields only marginal gains on MTI validation: F1 increases from $0.6775$ with the Standard encoder (${\sim}14$K parameters) to $0.6849$ with X-Large (${\sim}909$K), while XX-Large (${\sim}3.6$M) gives no further improvement.
This motivates using X-Large as the local encoder and relying on pair-level aggregation for cross-site evidence.

\Cref{tab:arch_ablation} summarizes the key aggregator ablations.
SAB-based Set Transformer outperforms a sorted 1D-CNN and a heavier GNN, while capacity scaling shows a sweet spot around $d_{\mathrm{model}}{=}1024$, $L{=}4$; ISAB underperforms, suggesting that inducing-point compression is too aggressive at the tested budget.

\begin{table}[!htbp]
\centering
\caption{
\textbf{Key architecture ablations on MTI validation.}
Representative configurations used to choose the final architecture; full sweeps are in Appendix~\ref{app:arch_ablation}.
}
\label{tab:arch_ablation}
\small
\setlength{\tabcolsep}{4pt}
\renewcommand{\arraystretch}{1.05}
\begin{tabular}{llcc}
\toprule
\textbf{Ablation} & \textbf{Variant} & \textbf{Setting} & \textbf{Val F1} \\
\midrule
\multicolumn{4}{l}{\textit{Aggregator family} ($K{=}64$, X-Large encoder)} \\
Type & CNN & sorted 1D-CNN & 0.6850 \\
Type & GNN & $k$-NN + GAT & 0.7602 \\
Type & \textbf{SAB} & \textbf{Set Transformer} & \textbf{0.7715} \\
\midrule
\multicolumn{4}{l}{\textit{Set Transformer capacity}} \\
Scale & Small & $d{=}256,L{=}2,H{=}8$ & 0.7273 \\
Scale & Sweet spot & $d{=}1024,L{=}4,H{=}16$ & \textbf{0.7353} \\
Scale & Wider & $d{=}1280,L{=}4,H{=}16$ & 0.7333 \\
Scale & Deeper & $d{=}1024,L{=}5,H{=}16$ & 0.7335 \\
Scale & ISAB & $d{=}512,L{=}3,H{=}8$ & 0.6417 \\
\bottomrule
\end{tabular}
\end{table}

\subsection{Cross-domain generalization}
\label{subsec:exp_cross_domain}

To assess whether BR-MIL generalizes beyond miRNA targeting, we evaluate on two standard MIL benchmarks from different domains (Table~\ref{tab:cross_domain}); implementation details are provided in Appendix~\ref{app:cross_domain}.
\begin{table}[!htbp]
  \centering
  \caption{
  \textbf{Cross-domain generalization.}
  \BRMIL{} is evaluated on CAMELYON16 (pathology) and Musk2 (molecular activity), both outperforming matched ABMIL baselines.
  $\dagger$ denotes published results under the same protocol from torchmil~\cite{torchmil}.
  SmTransformerABMIL~\cite{smmil} exploits spatial patch coordinates, a pathology-specific prior unavailable in general MIL settings.
  }
  \label{tab:cross_domain}
  \small
  \setlength{\tabcolsep}{4pt}
  \renewcommand{\arraystretch}{1.05}
  \begin{tabular}{lcrr}
    \toprule
    \textbf{Method} & \textbf{K} & \textbf{AUC} & \textbf{F1} \\
    \midrule
    \multicolumn{4}{l}{\textit{CAMELYON16 (pathology MIL; mean 8{,}764 patches/bag)}} \\
    ABMIL (ours) & --- & 0.968 {\scriptsize$\pm$0.035} & 0.952 {\scriptsize$\pm$0.023} \\
    \textbf{\BRMIL{} (ours)} & \textbf{1024} & \textbf{0.980} {\scriptsize$\pm$0.025} & \textbf{0.966} {\scriptsize$\pm$0.033} \\
    SmTransABMIL$^\dagger$~\cite{smmil} & --- & 0.982 {\scriptsize$\pm$0.006} & --- \\
    TransMIL$^\dagger$ & --- & 0.977 {\scriptsize$\pm$0.007} & --- \\
    DTFD-MIL$^\dagger$ & --- & 0.976 {\scriptsize$\pm$0.014} & --- \\
    \midrule
    \multicolumn{4}{l}{\textit{Musk2 (molecular activity MIL; mean 12 instances/bag; 10-fold CV $\times$ 3 seeds)}} \\
    ABMIL (matched capacity) & --- & 0.979 {\scriptsize$\pm$0.017} & 0.756 {\scriptsize$\pm$0.032} \\
    \textbf{\BRMIL{} (ours)} & \textbf{4} & \textbf{0.987} {\scriptsize$\pm$0.010} & \textbf{0.811} {\scriptsize$\pm$0.046} \\
    \bottomrule
  \end{tabular}
\end{table}

\paragraph{Results.}
On CAMELYON16, \BRMIL{} selects $K{=}1{,}024$ patches, about 11.7\% of the mean bag size, and improves over our matched ABMIL baseline in both AUC ($0.980$ vs.\ $0.968$) and F1 ($0.966$ vs.\ $0.952$).
Under the same torchmil protocol, it nearly matches SmTransformerABMIL ($0.980$ vs.\ $0.982$ AUC) despite not using spatial patch coordinates, and outperforms quoted permutation-invariant baselines without spatial priors.
On Musk2, \BRMIL{} with $K{=}4$ improves F1 over matched-capacity ABMIL by $+5.5$pp, suggesting that budgeted selection can help beyond compute reduction.

\paragraph{Limitations.}
\PAIRFormer{} relies on candidate generation, so functional sites missed by ESA filtering are unrecoverable.
Current evaluations are constrained by available pair-level datasets and model one miRNA--mRNA pair at a time, leaving noisy clinical-scale data and multi-miRNA co-regulation for future work.

\section{Conclusion}
\label{sec:conclusion}

We introduced \BRMIL{}, a budgeted relational MIL formulation for large, heavy-tailed bags, and instantiated it with \PAIRFormer{} for functional miRNA--mRNA target prediction.
By combining cheap full-pool scanning with expensive Set Transformer aggregation over only $K$ selected candidate sites, \PAIRFormer{} makes transcript-scale relational prediction practical.
Experiments on miRAW, deepTargetPro, and the large-scale MTI benchmark demonstrate strong performance and a controlled accuracy--compute tradeoff, while results on CAMELYON16 and Musk2 suggest broader applicability beyond miRNA sequence modeling.
Future work includes improving non-canonical \CTS{} discovery, extending to multi-miRNA regulation, and scaling budgeted aggregation with differentiable selection or sparse attention.

{\small
\bibliographystyle{unsrtnat}
\bibliography{reference}
}

\clearpage
\appendix

\section{Extended Problem Setup and BR-MIL Formalization}
\label{app:problem_setup_long}

\subsection{Biological instantiation: miRNA--mRNA functional targeting}
\label{sec:bio-instantiation}

\paragraph{Task.}
We study \emph{functional} miRNA--mRNA targeting.
Each example corresponds to a miRNA sequence $\mu^{(m)}$ and an mRNA \threeprimeUTR $\nu^{(m)}$,
with a transcript-level label $Y^{(m)}\in\{0,1\}$ indicating whether $\mu^{(m)}$ functionally represses $\nu^{(m)}$.
(We use $\mu,\nu$ to avoid overloading $u_i$ in the abstract formulation.)

\paragraph{Candidate target sites (CTSs).}
Given $(\mu^{(m)},\nu^{(m)})$, we extract a variable-size set of candidate target sites (CTSs)
from the 3'UTR, e.g., by sliding windows plus seed/alignment-based filtering:
\[
\mathcal{C}^{(m)}=\{c^{(m)}_1,\dots,c^{(m)}_{n^{(m)}}\}, \qquad n^{(m)}:=|\mathcal{C}^{(m)}|.
\]
Each CTS corresponds to a local window on $\nu^{(m)}$ that is potentially bindable by $\mu^{(m)}$.
The candidate pool size $n^{(m)}$ varies substantially across pairs.

\paragraph{MIL view and mapping to our abstract bags.}
We cast transcript-level prediction as multi-instance learning (MIL):
the $m$-th miRNA--mRNA pair forms a bag of CTS instances.
Concretely, each CTS $c^{(m)}_i$ induces
(i) an \emph{instance content} $x_i\in\mathcal{X}$ (e.g., the CTS window sequence/context features defined by $\mu^{(m)}$ and $\nu^{(m)}$),
and (ii) \emph{structural attributes} $u_i\in\mathcal{U}$ (e.g., normalized transcript position/region, seed type, alignment quality).
This yields the abstract bag
\[
B^{(m)}=\{(x_i,u_i)\}_{i=1}^{n^{(m)}},
\]
with bag label $Y^{(m)}$ and latent instance labels (whether a CTS is truly functional).
We next formalize this setting as \emph{Budgeted Relational MIL (BR-MIL)} under a strict per-bag budget $K$.

\subsection{Notation}
\label{sec:notation}
We observe a dataset of bags
\[
\mathcal{D}=\{(B^{(m)},Y^{(m)},\mathcal{L}^{(m)})\}_{m=1}^M .
\]

Each bag $B^{(m)}$ corresponds to one miRNA--mRNA pair $(\mu^{(m)},\nu^{(m)})$ introduced in Sec.~\ref{sec:bio-instantiation},
and its instances are the candidate target sites (CTSs) extracted from $\nu^{(m)}$ for $\mu^{(m)}$.

Each bag is a (multi)set of instances
\[
B=\{(x_i,u_i)\}_{i=1}^{n},
\]
where $x_i\in\mathcal{X}$ denotes instance content (e.g., CTS windows) and
$u_i\in\mathcal{U}$ denotes structural attributes (e.g., transcript position/region).
The bag label is $Y\in\{0,1\}$ (or $Y\in\mathbb{R}$ for regression).
Optionally, a subset of instances is labeled: $\mathcal{L}\subseteq[n]$, with instance labels
$\{\tilde y_i\}_{i\in\mathcal{L}}$.

A \emph{cheap encoder} $\tilde e_{\tilde\theta}:\mathcal{X}\times\mathcal{U}\to\mathbb{R}^{d_0}$
produces cheap representations $\tilde h_i=\tilde e_{\tilde\theta}(x_i,u_i)$ and the multiset
$\tilde H=\{\tilde h_i\}_{i=1}^{n}$.
An \emph{expensive encoder} $e_{\theta}:\mathcal{X}\times\mathcal{U}\to\mathbb{R}^{d}$
produces $h_i=e_{\theta}(x_i,u_i)$, but can be evaluated for at most $K$ instances per bag.
A \emph{selector} $\pi_\psi(\cdot\mid \tilde H)$ outputs a subset $S\subseteq[n]$ with $|S|\le K$.

Pairwise relations among selected instances may be provided explicitly as
$R_S=\{r_{ij}\}_{i,j\in S}$ with
\[
r_{ij}=\rho(h_i,h_j,u_i,u_j),
\]
or modeled implicitly by the aggregator.
A permutation-invariant \emph{relational aggregator} $f_\phi$ outputs the bag prediction
\begin{equation}
\hat Y
= f_\phi\!\big(\{(h_i,u_i)\}_{i\in S},\, R_S\big).
\label{eq:predictor}
\end{equation}
When instance labels are available, an instance head $c_\omega$ yields
$\hat y_i=c_\omega(h_i)$ for $i\in S\cap\mathcal{L}$.

\begin{definition}[Budgeted Relational MIL (BR-MIL)]
\label{def:brmil}
\textbf{Budgeted Relational MIL (BR-MIL)} is a supervised learning problem characterized by:
(i) large candidate pools ($n$ per bag), (ii) a strict per-bag budget $K$ on expensive encoding and
relational aggregation, and (iii) modeling interaction effects among selected instances.

Formally, the hypothesis class consists of triplets $(\tilde e_{\tilde\theta},\pi_\psi,f_\phi)$
with prediction
\begin{equation}
\begin{aligned}
\hat Y(B)
&=
f_\phi\!\Big(\{(e_\theta(x_i,u_i),u_i)\}_{i\in S},\, R_S\Big),\\
S
&\sim \pi_\psi\!\Big(\cdot\ \Big|\ \{\tilde e_{\tilde\theta}(x_i,u_i)\}_{i=1}^{n}\Big),
\qquad |S|\le K .
\end{aligned}
\label{eq:brmil-hypothesis}
\end{equation}

Training minimizes the expected risk with mixed supervision:
\begin{equation}
\begin{aligned}
&\min_{\tilde\theta,\psi,\phi,\omega}\\
&\mathbb{E}_{(B,Y,\mathcal{L})\sim \mathcal{D}}
\ \mathbb{E}_{S\sim \pi_\psi(\cdot\mid \tilde H)}
\Big[
\ell_{\text{bag}}(Y,\hat Y)
+\lambda \!\!\sum_{i\in S\cap\mathcal{L}}\!\! \ell_{\text{inst}}(\tilde y_i,\hat y_i)
\Big]\\
&\quad +\Omega(\pi_\psi),
\end{aligned}
\label{eq:objective}
\end{equation}

where $\Omega$ regularizes selection (e.g., diversity or exploration).
The budget constraint is hard: expensive encoding and relational aggregation are executed only on $S$.

\end{definition}

\subsection{Permutation Invariance}
\label{sec:invariance}
A BR-MIL predictor is \emph{bag-permutation invariant} if for any permutation $\pi$ of indices,
\begin{equation}
\hat Y\big(\{(x_{\pi(i)},u_{\pi(i)})\}_{i=1}^{n}\big)=\hat Y\big(\{(x_i,u_i)\}_{i=1}^{n}\big).
\label{eq:perm-inv}
\end{equation}
Sufficient conditions are:
(i) the selector depends only on the multiset $\tilde H$ (not on order) and outputs $S$
via scores with either no ties or a fixed, data-independent tie-breaking rule;
(ii) the aggregator $f_\phi$ is permutation invariant w.r.t.\ its selected inputs (e.g., DeepSets or Set Transformer);
(iii) if explicit relations $R_S$ are constructed, the relation function $\rho$ is permutation-equivariant
(i.e., permuting selected indices permutes $R_S$ consistently).
In particular, Top-$K$ selection preserves invariance under (i)--(iii) because the prediction depends only on the set $S$.

\subsection{Computational Complexity}
\label{sec:complexity}
For each bag, the computational cost decomposes into two independent bottlenecks:

\textbf{(a) Per-instance encoding cost.}
Running the expensive encoder $e_\theta$ on all $n$ candidates would require $\mathcal{O}(n)$ expensive forward passes.
For heavy-tailed $n$ (e.g., $n>500$; see Figure~\ref{fig:n_distribution} for empirical distribution), this becomes prohibitive even if aggregation were free.
BR-MIL avoids this by using a cheap encoder $\tilde{e}_{\tilde{\theta}}$ to scan all $n$ candidates in $\mathcal{O}(n)$ cheap operations, then applying $e_\theta$ only to the selected $K$ instances, reducing expensive encoding to $\mathcal{O}(K)$.

\textbf{(b) Relational aggregation cost.}
Attention-based aggregators (e.g., Set Transformer) scale as $\mathcal{O}(n^2)$ in the number of encoded instances due to pairwise interactions.
Even if encoding were free, $\mathcal{O}(n^2)$ aggregation is intractable for large $n$.
BR-MIL restricts aggregation to the $K$-element selected subset, reducing this to $\mathcal{O}(K^2)$.

Thus the deployed cost per bag is
\[
\mathcal{O}(n)_{\text{cheap scan}} + \mathcal{O}(K)_{\text{expensive encoding}} + \mathcal{O}(K^2)_{\text{aggregation}},
\]
under a strict access budget $K$, versus $\mathcal{O}(n)_{\text{expensive encoding}} + \mathcal{O}(n^2)_{\text{aggregation}}$ if all instances were processed relationally.

Our proposed \PAIRFormer{} instantiates this BR-MIL hypothesis class with
(i) a TargetNet-compatible expensive encoder,
(ii) a distilled cheap encoder for full-pool scanning,
(iii) a deterministic budgeted selector (STSelector), and
(iv) a Set Transformer aggregator.

\section{Dataset and Candidate Pool Statistics}
\label{app:dataset_stats}

\subsection{Candidate Pool Size Statistics}
\label{app:n_distribution_stats}

Table~\ref{tab:n_distribution_stats_app} reports detailed statistics of the number of valid \CTSs{} per miRNA--mRNA pair after ESA filtering.
These statistics complement Fig.~\ref{fig:n_distribution} in the main text and quantify the heavy-tailed candidate-pool regime motivating budgeted selection.

\begin{table}[!htbp]
\centering
\caption{
\textbf{Candidate pool size statistics across benchmarks.}
Statistics are computed after ESA filtering ($s_i^{\mathrm{esa}}\ge 6$).
The operating budget in the main experiments is $K{=}64$.
}
\label{tab:n_distribution_stats_app}
\small
\setlength{\tabcolsep}{5pt}
\begin{tabular}{lrrrrrrr}
\toprule
Dataset & Pairs & Mean & Median & P95 & P99 & Max & Pairs $>K{=}64$ \\
\midrule
miRAW & 10,960 & 1,365 & 912 & 4,067 & 6,897 & 24,983 & 94.7\% \\
deepTargetPro & 8,073 & 1,369 & 993 & 3,927 & 6,128 & 11,071 & 95.2\% \\
MTI & 420,351 & 1,157 & 640 & 3,853 & 7,214 & 29,894 & 92.4\% \\
\bottomrule
\end{tabular}
\end{table}

\section{Full Theoretical Assumptions and Proofs}
\label{app:theory}

\subsection{Preliminaries: full-pool and masked predictors}
\label{app:theory:prelim}

We formalize the comparison between a conceptual full-pool relational predictor and the deployed
budgeted predictor. Fix a bag
\[
B=\{(x_i,u_i)\}_{i=1}^{n},
\]
and define the expensive token of instance $i$ as
\[
z_i := z(x_i,u_i)\in\mathbb{R}^{d_z},
\]
where $z(\cdot)$ denotes the composition of the expensive encoder and tokenization
(e.g., Eq.~\eqref{eq:token}). Let $Z(B):=(z_1,\dots,z_n)$.

\paragraph{Reference full-information predictor.}
Let $f_\phi$ be a permutation-invariant relational aggregator operating on a multiset of tokens.
The reference predictor (conceptually) accesses all $n$ expensive tokens:
\begin{equation}
\hat Y_{\mathrm{full}}(B) := f_\phi\big(\{z_i\}_{i=1}^{n}\big).
\label{eq:app_y_full}
\end{equation}

\paragraph{Budgeted predictor via masking.}
A selector outputs a subset $S(B)\subseteq[n]$ with $|S(B)|\le K$.
Fix a padding token $z_\emptyset\in\mathbb{R}^{d_z}$ and define the masked token list
\[
Z^{\mathrm{mask}(S)}(B):=(z_1',\dots,z_n'),\qquad
z_i' =
\begin{cases}
z_i,& i\in S,\\
z_\emptyset,& i\notin S.
\end{cases}
\]
The budgeted prediction is
\begin{equation}
\hat Y_{K}(B) := f_\phi\big(\{z_i'\}_{i=1}^{n}\big)
= f_\phi\big(\{z_i\}_{i\in S(B)};\; \text{mask}\big).
\label{eq:app_y_k}
\end{equation}
In practice (Set Transformer / SAB), the ``mask'' is implemented by attention masks so that padded tokens
do not contribute to attention weights or pooling, making \eqref{eq:app_y_k} a faithful abstraction.

\subsection{Approximation: influence tail and selector regret}
\label{app:theory:approx}

The approximation result separates two effects: the unavoidable tail mass left by any size-$K$
subset, and the regret of the actual selector relative to that oracle subset. This avoids placing
the desired coverage property directly into an assumption.

\begin{assumption}[Mask-consistent execution]\label{assm:A0}
There exists a fixed padding token $z_\emptyset$ and a masking convention such that
the implemented budgeted forward pass (padding to \code{kmax} and masking in attention/pooling)
is equivalent to evaluating $f_\phi$ on the masked token multiset
$\{z_i'\}_{i=1}^{n}$ defined above.
\end{assumption}

\begin{assumption}[Bounded token radius]\label{assm:A1}
There exists $R>0$ such that for all bags and all instances,
$\|z_i\|_2\le R$ and $\|z_\emptyset\|_2\le R$.
\end{assumption}

\begin{assumption}[Uniform masking sensitivity]\label{assm:A2}
There exists $L_\star<\infty$ such that for every bag $B$ there are sensitivity scores
$a_i(B)\ge 0$ with $A(B):=\sum_{i=1}^{n}a_i(B)\le L_\star$ and, for every subset
$S\subseteq[n]$,
\begin{equation}
|\hat Y_{\mathrm{full}}(B)-f_\phi(Z^{\mathrm{mask}(S)}(B))|
\;\le\;
\sum_{i\notin S} a_i(B)\,\|z_i-z_\emptyset\|_2 .
\label{eq:A2_sensitivity}
\end{equation}
\end{assumption}

The scores $a_i(B)$ should be read as masking influences: they upper bound how much the
prediction can change when token $i$ is replaced by the padding token while the remaining
masked context varies. The next lemma gives a checkable sufficient condition rather than
assuming the desired approximation bound directly.

\begin{lemma}[Gradient-based sensitivities imply Assumption~\ref{assm:A2}]
\label{lem:grad_weights}
Assume $f_\phi$ is differentiable in each token on every line segment between a masked state and a
state obtained by unmasking one token. Let $\mathcal{Z}(B)$ denote the set of all masked states
obtained from $Z(B)$, and define
\[
g_i(B):=\sup_{Z\in\mathcal{Z}(B)}\big\|\nabla_{z_i} f_\phi(Z)\big\|_2 .
\]
If $\sum_i g_i(B)\le L_\star$ for all bags, then Assumption~\ref{assm:A2} holds with
$a_i(B)=g_i(B)$.
\end{lemma}

\begin{proof}
Let $Z^{(0)}:=Z(B)$ and obtain $Z^{(t)}$ by masking one additional token at each step until reaching
$Z^{\mathrm{mask}(S)}(B)$ (mask exactly those $i\notin S$), so that only one token changes per step.
By the mean value theorem applied to the $i$-th token at step $t$,
\[
\begin{aligned}
\big|f_\phi(Z^{(t-1)})-f_\phi(Z^{(t)})\big|
&\le
\sup_{Z\in\mathcal{Z}(B)}\|\nabla_{z_i} f_\phi(Z)\|_2\,\|z_i-z_\emptyset\|_2
\\&=
g_i(B)\,\|z_i-z_\emptyset\|_2.
\end{aligned}
\]
Summing over all masked tokens $i\notin S$ yields the bound.
\end{proof}

\begin{definition}[Influence tail and selector regret]\label{def:influence_tail}
Assume $A(B)>0$ and define normalized influence weights
\begin{equation}
w_i(B):=\frac{a_i(B)}{A(B)},\qquad \sum_{i=1}^{n}w_i(B)=1 .
\end{equation}
If $A(B)=0$, the full and masked predictors are identical under Assumption~\ref{assm:A2}, and the
choice of $w$ is immaterial. Let $k_B:=\min(K,n)$ and let
$w_{(1)}(B)\ge\cdots\ge w_{(n)}(B)$ denote the sorted weights. The oracle top-$K$ tail mass is
\begin{equation}
\psi_K(B):=1-\sum_{j=1}^{k_B}w_{(j)}(B).
\label{eq:oracle_tail}
\end{equation}
For a possibly randomized selector $S(B)$, its influence regret is
\begin{equation}
\Delta^{\mathrm{w}}_K(B):=
\sum_{j=1}^{k_B}w_{(j)}(B)
-
\mathbb{E}\!\left[\sum_{i\in S(B)}w_i(B)\,\middle|\,B\right]\ge 0 .
\label{eq:selector_influence_regret}
\end{equation}
We write
\begin{equation}
\varepsilon_K(B):=\psi_K(B)+\Delta^{\mathrm{w}}_K(B).
\label{eq:epsilon_decomp}
\end{equation}
\end{definition}

\begin{lemma}[Best-$K$ mass equals the sorted top-$K$ sum]\label{lem:bestK_mass}
For the weights of a fixed bag $B$, with $k_B=\min(K,n)$ and $\sum_i w_i=1$,
\[
\max_{|S|\le k_B}\ \sum_{i\in S} w_i \;=\; \sum_{j=1}^{k_B} w_{(j)}.
\]
Consequently, the minimal uncovered mass achievable by any budget-$K$ subset equals $\psi_K(B)$.
\end{lemma}

\begin{proof}
The maximum is achieved by selecting the $k_B$ largest weights; any other subset can be improved
by exchanging a smaller selected weight with a larger unselected weight. The uncovered mass is
$1-\sum_{i\in S}w_i$, so its minimum is $1-\sum_{j=1}^{k_B} w_{(j)}=\psi_K(B)$.
\end{proof}

\begin{theorem}[Formal approximation bound under budgeted masking]
\label{thm:approx_formal}
Assume Assumptions~\ref{assm:A0}--\ref{assm:A2}.
Then for any fixed bag $B$,
\begin{equation}
\mathbb{E}\Big[\,\big|\hat Y_{\mathrm{full}}(B)-\hat Y_{K}(B)\big|\,\Big|\ B\Big]
\;\le\;
2R\,A(B)\,\varepsilon_K(B)
\;\le\;
2R\,L_\star\,\varepsilon_K(B),
\label{eq:approx_bound}
\end{equation}
where the expectation is over selector randomness and
$\varepsilon_K(B)=\psi_K(B)+\Delta_K^{\mathrm{w}}(B)$.
\end{theorem}

\begin{proof}
By Assumption~\ref{assm:A0} and the definition of $\hat Y_K$ in \eqref{eq:app_y_k},
\[
\hat Y_K(B)= f_\phi(\{z_i'\}_{i=1}^{n})
\quad\text{with}\quad
z_i'=
\begin{cases}
z_i,& i\in S(B),\\
z_\emptyset,& i\notin S(B).
\end{cases}
\]
Applying Assumption~\ref{assm:A2} to the subset $S(B)$ and using Assumption~\ref{assm:A1} gives
\[
\begin{aligned}
\big|\hat Y_{\mathrm{full}}(B)-\hat Y_{K}(B)\big|
\le
\sum_{i\notin S(B)}a_i(B)\|z_i-z_\emptyset\|_2
\\
\le
2R\,A(B)\sum_{i\notin S(B)}w_i(B)
\\=
2R\,A(B)\Big(1-\sum_{i\in S(B)} w_i(B)\Big).
\end{aligned}
\]
Taking conditional expectation over selector randomness,
\[
\mathbb{E}\!\left[1-\sum_{i\in S(B)}w_i(B)\,\middle|\,B\right]
=
1-\sum_{j=1}^{k_B}w_{(j)}(B)+\Delta_K^{\mathrm{w}}(B)
=\psi_K(B)+\Delta_K^{\mathrm{w}}(B).
\]
The first inequality in \eqref{eq:approx_bound} follows, and the second uses $A(B)\le L_\star$.
\end{proof}

\begin{proposition}[Cheap proxy quality controls selector regret]
\label{prop:proxy_regret}
Let $q_i(B)$ be any cheap proxy score used by a selector, and let
\[
T_q^\star(B)\in\arg\max_{|T|\le k_B}\sum_{i\in T}q_i(B).
\]
Define the proxy regret
\[
\rho_K^q(B):=
\sum_{i\in T_q^\star(B)}q_i(B)
-
\mathbb{E}\!\left[\sum_{i\in S(B)}q_i(B)\,\middle|\,B\right]\ge 0 .
\]
If $\max_i |q_i(B)-w_i(B)|\le \eta(B)$, then
\begin{equation}
\Delta^{\mathrm{w}}_K(B)\le \rho_K^q(B)+2k_B\eta(B)
\le \rho_K^q(B)+2K\eta(B).
\label{eq:proxy_regret_bound}
\end{equation}
In particular, exact top-$K$ selection under a uniformly accurate proxy has
$\Delta^{\mathrm{w}}_K(B)\le 2K\eta(B)$.
\end{proposition}

\begin{proof}
Let $T_w^\star$ be a top-$k_B$ set under $w$. Since $T_q^\star$ maximizes proxy mass,
$\sum_{i\in T_w^\star}q_i\le \sum_{i\in T_q^\star}q_i$. Therefore
\[
\begin{aligned}
\Delta^{\mathrm{w}}_K(B)
&=
\sum_{i\in T_w^\star}w_i
-
\mathbb{E}\!\left[\sum_{i\in S(B)}w_i\,\middle|\,B\right]
\\
&\le
\sum_{i\in T_w^\star}(w_i-q_i)
+
\sum_{i\in T_q^\star}q_i
-
\mathbb{E}\!\left[\sum_{i\in S(B)}q_i\,\middle|\,B\right]
\\&\qquad
+
\mathbb{E}\!\left[\sum_{i\in S(B)}(q_i-w_i)\,\middle|\,B\right]
\\
&\le k_B\eta(B)+\rho_K^q(B)+k_B\eta(B).
\end{aligned}
\]
\end{proof}

\begin{corollary}[Cheap-score top-$K$ selector]\label{cor:topk_proxy}
Let $S_q(B)$ select the $k_B$ largest proxy scores $q_i(B)$, with deterministic tie-breaking.
If $\max_i |q_i(B)-w_i(B)|\le \eta(B)$, then
\[
\Delta_K^{\mathrm{w}}(B)\le 2k_B\eta(B)\le 2K\eta(B).
\]
Moreover, if $k_B<n$ and the influence margin
\[
\gamma_K(B):=w_{(k_B)}(B)-w_{(k_B+1)}(B)
\]
satisfies $\eta(B)<\gamma_K(B)/2$, then $S_q(B)$ is an oracle top-$K$ influence set and
$\Delta_K^{\mathrm{w}}(B)=0$. If $k_B=n$, the tail and regret are both zero.
\end{corollary}

\begin{proof}
For $S=S_q$, the proxy regret $\rho_K^q(B)$ in Proposition~\ref{prop:proxy_regret} is zero,
which gives the first claim. For the margin claim, every true top-$k_B$ element $i$ and every
non-top-$k_B$ element $j$ satisfy
\[
q_i \ge w_{(k_B)}-\eta > w_{(k_B+1)}+\eta \ge q_j.
\]
Thus cheap-score top-$K$ recovers an oracle top-$K$ influence set.
\end{proof}

\begin{corollary}[Influence compressibility gives a tail rate]\label{cor:tail_rate}
If the sorted influence weights obey $w_{(j)}(B)\le C_{\mathrm{tail}}(B)j^{-\alpha}$ for some
$\alpha>1$, then
\[
\psi_K(B)
\le
\frac{C_{\mathrm{tail}}(B)}{\alpha-1}\,k_B^{1-\alpha}
\quad\text{for } k_B<n,
\]
and $\psi_K(B)=0$ for $k_B=n$. Consequently,
\[
\mathbb{E}\Big[|\hat Y_{\mathrm{full}}(B)-\hat Y_K(B)|\,\Big|\,B\Big]
\le
2R L_\star
\left(
\frac{C_{\mathrm{tail}}(B)}{\alpha-1}\,k_B^{1-\alpha}
\;+\;
\Delta_K^{\mathrm{w}}(B)
\right).
\]
If the selector is cheap-score top-$K$, Corollary~\ref{cor:topk_proxy} further replaces
$\Delta_K^{\mathrm{w}}(B)$ by $2K\eta(B)$.
\end{corollary}

\begin{proof}
For $k_B<n$,
\[
\psi_K(B)=\sum_{j=k_B+1}^{n}w_{(j)}(B)
\le
C_{\mathrm{tail}}(B)\sum_{j=k_B+1}^{\infty}j^{-\alpha}
\le
\frac{C_{\mathrm{tail}}(B)}{\alpha-1}k_B^{1-\alpha}.
\]
Substitution into Theorem~\ref{thm:approx_formal} gives the prediction-gap bound.
\end{proof}

\paragraph{Connection to Theorem~\ref{thm:approx} in the main text.}
Theorem~\ref{thm:approx} summarizes Theorem~\ref{thm:approx_formal}. The approximation gap is
small when (i) the oracle influence distribution has a light enough top-$K$ tail, and (ii) the
selector has low influence regret, which can be achieved by a cheap score that is well aligned with
the true masking influence. The theory is selector-agnostic: simple cheap-score top-$K$ is covered
by Corollary~\ref{cor:topk_proxy}, while more elaborate deterministic selectors used in an
implementation are covered through the same regret term.

\subsection{Capacity control for masked relational aggregation}
\label{app:theory:capacity}

The generalization theorem below uses a capacity envelope for the expensive relational class.
The purpose is not to claim that an unconstrained attention network automatically satisfies a
dimension-free bound, but to isolate the mechanism that matters for BR-MIL: once the selector is
fixed, the relational learner receives a padded matrix with at most $K$ non-padding tokens, so its
input radius is $O(\sqrt{K})$ and the raw candidate-pool size $n$ is absent. The proposition below
gives a sufficient condition under which the envelope holds.

\begin{assumption}[Capacity envelope for the relational class]\label{assm:Gcap}
For every fixed selector $S$ with $|S(B)|\le K$, define
\[
\mathcal{G}_{K,S}:=\{B\mapsto f(X_S(B)):\ f\in\mathcal{F}_K\},
\]
where $X_S(B)\in\mathbb{R}^{K\times d_z}$ is the padded selected-token matrix and
$\mathfrak{R}_M$ denotes expected Rademacher complexity. Under the token-radius bound in
Assumption~\ref{assm:Ginput}, there is a constant $C_{\mathrm{rel}}$ independent of $K$, $n$, and
$M$ such that
\begin{equation}
\mathfrak{R}_M(\mathcal{G}_{K,S})
\le
C_{\mathrm{rel}}\frac{R\sqrt{K}}{\sqrt{M}} .
\label{eq:capacity_envelope}
\end{equation}
\end{assumption}

\begin{proposition}[Feature-map sufficient condition for Assumption~\ref{assm:Gcap}]
\label{prop:gcap_feature}
Fix a selector $S$. Suppose every $f\in\mathcal{F}_K$ can be written as
\[
f_{\theta,u}(X)=u^\top \Phi_\theta(X),
\qquad \|u\|_2\le B_{\mathrm{out}},
\]
where $\Phi_\theta$ is a permutation-invariant masked-token representation. If, for every sample
$\{B^{(m)}\}_{m=1}^{M}$,
\begin{equation}
\mathbb{E}_{\varepsilon}\!\left[
\sup_{\theta}
\left\|
\frac{1}{M}\sum_{m=1}^{M}\varepsilon_m\,
\Phi_\theta\!\left(X_S(B^{(m)})\right)
\right\|_2
\right]
\le
C_{\Phi}\frac{R\sqrt{K}}{\sqrt{M}},
\label{eq:feature_rad_condition}
\end{equation}
with $C_{\Phi}$ independent of $K$, $n$, and $M$, then Assumption~\ref{assm:Gcap} holds with
$C_{\mathrm{rel}}=B_{\mathrm{out}}C_{\Phi}$.
In particular, \eqref{eq:feature_rad_condition} holds for a fixed representation map satisfying
$\|\Phi(X_S(B))\|_2\le C_{\Phi}R\sqrt{K}$.
\end{proposition}

\begin{proof}
For the class $\mathcal{G}_{K,S}$,
\[
\begin{aligned}
\mathfrak{R}_M(\mathcal{G}_{K,S})
&=
\mathbb{E}_{\mathcal{D},\varepsilon}
\left[
\sup_{\theta,\|u\|_2\le B_{\mathrm{out}}}
\frac{1}{M}\sum_{m=1}^{M}
\varepsilon_m u^\top\Phi_\theta(X_S(B^{(m)}))
\right]
\\
&\le
B_{\mathrm{out}}\,
\mathbb{E}_{\mathcal{D},\varepsilon}
\left[
\sup_\theta
\left\|
\frac{1}{M}\sum_{m=1}^{M}
\varepsilon_m\Phi_\theta(X_S(B^{(m)}))
\right\|_2
\right]
\\
&\le
B_{\mathrm{out}}C_\Phi\frac{R\sqrt{K}}{\sqrt{M}}.
\end{aligned}
\]
For a fixed $\Phi$, Jensen's inequality gives
\[
\mathbb{E}_{\varepsilon}\left\|
\frac{1}{M}\sum_{m=1}^{M}\varepsilon_m\Phi(X_S(B^{(m)}))
\right\|_2
\le
\frac{1}{M}\left(\sum_{m=1}^{M}\|\Phi(X_S(B^{(m)}))\|_2^2\right)^{1/2}
\le
C_\Phi\frac{R\sqrt{K}}{\sqrt{M}}.
\]
\end{proof}

\begin{remark}[Set Transformer instantiation]
A fixed-depth Set Transformer with fixed width/head count can be treated through
Proposition~\ref{prop:gcap_feature} when its linear maps, seed vectors, normalization gains,
attention logits, and final readout are norm controlled. Standard norm-based Rademacher bounds for
neural networks \citep{bartlett2017spectrally} control the representation term
\eqref{eq:feature_rad_condition}; depth, width, head count, and norm products are absorbed into
$C_\Phi$ and $B_{\mathrm{out}}$. Attention masks remove padded tokens, so these constants do not
depend on the unselected pool size $n$.
\end{remark}

\subsection{Generalization controlled by the budget $K$ (Theorem~\ref{thm:gen})}
\label{app:theory:gen}

We provide a formal generalization result for the \emph{expensive relational component} in BR-MIL.
The key message is that once expensive computation is restricted to at most $K$ selected tokens per bag,
the dominant capacity term of the expensive relational predictor depends on the selected-token
radius $O(\sqrt{K})$, not on the raw pool size $n$. Any additional $n$-dependence can only enter
through the selector family.

\paragraph{Setup.}
Let $\mathcal{P}$ be a distribution over labeled bags $(B,Y)$.
We observe $M$ i.i.d.\ samples $\mathcal{D}=\{(B^{(m)},Y^{(m)})\}_{m=1}^M$.
For each bag $B$, a selector $S(\cdot)$ outputs $S(B)\subseteq[n]$ with $|S(B)|\le K$.
Let $X_S(B)\in\mathbb{R}^{K\times d_z}$ denote the selected tokens padded with $z_\emptyset$
tokens as needed, and let
\[
\mathrm{vec}(X_S(B)) \in \mathbb{R}^{K d_z}
\]
be its vectorization.

\begin{assumption}[i.i.d.\ bags]\label{assm:G0}
$(B^{(m)},Y^{(m)})$ are drawn i.i.d.\ from $\mathcal{P}$.
\end{assumption}

\begin{assumption}[Bounded selected-token radius]\label{assm:Ginput}
There exists $R>0$ such that, for any bag $B$, every selected token satisfies
$\|z_i(B)\|_2 \le R$ and the padding token satisfies $\|z_\emptyset\|_2 \le R$.
Equivalently, $\|\mathrm{vec}(X_S(B))\|_2 \le R\sqrt{K}$ for all selectors satisfying $|S(B)|\le K$.
\end{assumption}

\begin{assumption}[Bounded Lipschitz surrogate]\label{assm:Gloss}
The analyzed loss $\ell:\mathcal{Y}\times\mathbb{R}\to[0,1]$ is $L_\ell$-Lipschitz in its
prediction argument:
for all $y$ and all $a,b\in\mathbb{R}$,
$|\ell(y,a)-\ell(y,b)|\le L_\ell |a-b|$.
For BCE-style training, this assumption applies to a clipped-logit or clipped-loss surrogate used
for analysis.
\end{assumption}

For the uniform-selector result, let $\mathcal{S}_K$ be a finite family of selectors such that
$|S(B)|\le K$ for all $S\in\mathcal{S}_K$ and all bags $B$, and write $|\mathcal{S}_K|=N_S$.

\paragraph{Risk.}
For an aggregator $f$ and selector $S$, define the loss function
\[
h_{f,S}(B,Y) := \ell\big(Y, f(X_S(B))\big).
\]
Let the population and empirical risks be
\[
\begin{aligned}
\mathcal{R}(f,S):&=\mathbb{E}_{(B,Y)\sim\mathcal{P}}\big[h_{f,S}(B,Y)\big],
\\
\widehat{\mathcal{R}}(f,S):&=\frac{1}{M}\sum_{m=1}^M h_{f,S}(B^{(m)},Y^{(m)}).
\end{aligned}
\]

\paragraph{Rademacher complexity.}
For a class $\mathcal{H}$ of functions mapping examples to $[0,1]$, define its empirical Rademacher complexity as
\[
\widehat{\mathfrak{R}}_M(\mathcal{H})
:= \mathbb{E}_{\varepsilon}\Big[\sup_{h\in\mathcal{H}} \frac{1}{M}\sum_{m=1}^M \varepsilon_m\, h(B^{(m)},Y^{(m)})\Big],
\]
where $\varepsilon_1,\dots,\varepsilon_M$ are i.i.d.\ Rademacher variables.
Let $\mathfrak{R}_M(\mathcal{H}) := \mathbb{E}_{\mathcal{D}}[\widehat{\mathfrak{R}}_M(\mathcal{H})]$.

We consider two classes:
\[
\mathcal{H}_{K,S}
:=\{(B,Y)\mapsto \ell(Y,f(X_S(B))):\ f\in\mathcal{F}_K\}
~(\text{$S$ fixed}),
\]
and the union over selectors
\[
\begin{aligned}
\mathcal{H}_{K}
:=\bigcup_{S\in\mathcal{S}_K}\mathcal{H}_{K,S}
&=\{(B,Y)\mapsto \ell(Y,f(X_S(B))):\\ &f\in\mathcal{F}_K,\, S\in\mathcal{S}_K\}.
\end{aligned}
\]

\begin{theorem}[Generalization governed by budget $K$; formal version]\label{thm:gen_formal}
Assume \Cref{assm:G0,assm:Ginput,assm:Gloss,assm:Gcap}.
Fix any selector $S$ before observing the training labels, or condition on such a selector, and
consider $\mathcal{H}_{K,S}$.
Then for any $\delta\in(0,1)$, with probability at least $1-\delta$ over the draw of $\mathcal{D}$,
for all $f\in\mathcal{F}_K$,
\begin{equation}
\mathcal{R}(f,S)
\;\le\;
\widehat{\mathcal{R}}(f,S)
+ 2L_\ell C_{\mathrm{rel}}\frac{R\sqrt{K}}{\sqrt{M}}
+ 3\sqrt{\frac{\log(2/\delta)}{2M}}.
\label{eq:gen_base}
\end{equation}
Thus the dominant relational capacity term scales as $\mathcal{O}(\sqrt{K/M})$ and is independent
of the raw pool size $n$.

If selectors are chosen from a finite family $\mathcal{S}_K$ with $|\mathcal{S}_K|=N_S<\infty$,
then with probability at least $1-\delta$, uniformly for all $(f,S)\in\mathcal{F}_K\times\mathcal{S}_K$,
\begin{equation}
\mathcal{R}(f,S)
\;\le\;
\widehat{\mathcal{R}}(f,S)
+ 2L_\ell C_{\mathrm{rel}}\frac{R\sqrt{K}}{\sqrt{M}}
+ 3\sqrt{\frac{\log(2N_S/\delta)}{2M}}.
\label{eq:gen_unionS}
\end{equation}
\end{theorem}

\paragraph{Interpretation (where $n$ can enter).}
The bound isolates that the \emph{expensive relational component} depends on $K$ (and fixed token dimension),
not on $n$.
Any $n$-dependence can only enter through the selector family size $N_S$.
For instance, if $\mathcal{S}_K$ contains \emph{all} $K$-subsets of an $n$-candidate pool, then
$\log N_S \le \log\sum_{j=0}^K {n\choose j} = \mathcal{O}(K\log(en/K))$,
whereas for any deterministic selector with fixed cheap encoder, fixed hyperparameters, and fixed
tie-breaking, $N_S=1$ for this conditional analysis and the selector term vanishes.


\begin{lemma}[Standard Rademacher generalization bound]\label{lem:rad_gen}
For any $\mathcal{H}$ of functions into $[0,1]$ and any $\delta\in(0,1)$,
with probability at least $1-\delta$ over $\mathcal{D}$,
\[
\begin{aligned}
\forall h\in\mathcal{H}:\quad
\left|
\mathbb{E}_{(B,Y)\sim\mathcal{P}}\big[h(B,Y)\big]
-
\frac{1}{M}\sum_{m=1}^M h\big(B^{(m)},Y^{(m)}\big)
\right|
\le
2\,\mathfrak{R}_M(\mathcal{H})
+ 3\sqrt{\frac{\log(2/\delta)}{2M}}.
\end{aligned}
\]

\end{lemma}

\begin{proof}
This is a standard consequence of symmetrization, concentration (Hoeffding), and the definition of
Rademacher complexity (e.g., Bartlett--Mendelson style bounds).
\end{proof}

\begin{lemma}[Contraction for Lipschitz losses]\label{lem:contraction}
Assume \Cref{assm:Gloss}. Let $\mathcal{G}$ be a class of real-valued functions on bags, and define
$\ell\circ\mathcal{G}:=\{(B,Y)\mapsto \ell(Y,g(B)):\ g\in\mathcal{G}\}$.
Then
\[
\mathfrak{R}_M(\ell\circ\mathcal{G}) \le L_\ell\, \mathfrak{R}_M(\mathcal{G}).
\]
\end{lemma}

\begin{proof}
This is the standard contraction inequality for Rademacher averages applied pointwise in $Y$,
using the $L_\ell$-Lipschitz property in the prediction argument.
\end{proof}

\begin{proof}[Proof of \Cref{thm:gen_formal}]
Apply \Cref{lem:rad_gen} to $\mathcal{H}_{K,S}$.
For the complexity term, write $\mathcal{H}_{K,S} = \ell\circ \mathcal{G}_{K,S}$.
By \Cref{lem:contraction},
$\mathfrak{R}_M(\mathcal{H}_{K,S}) \le L_\ell \mathfrak{R}_M(\mathcal{G}_{K,S})$,
and by Assumption~\ref{assm:Gcap},
$\mathfrak{R}_M(\mathcal{G}_{K,S}) \le C_{\mathrm{rel}}R\sqrt{K/M}$.
Substituting gives \eqref{eq:gen_base}.

For the uniform bound over $(f,S)$ when $|\mathcal{S}_K|=N_S<\infty$,
apply \Cref{lem:rad_gen} to each class $\mathcal{H}_{K,S}$ with failure probability $\delta/N_S$.
By a union bound over $S\in\mathcal{S}_K$, with probability at least $1-\delta$,
simultaneously for all $S\in\mathcal{S}_K$ and all $f\in\mathcal{F}_K$,
\[
\mathcal{R}(f,S)
\le
\widehat{\mathcal{R}}(f,S)
+ 2\,\mathfrak{R}_M(\mathcal{H}_{K,S})
+ 3\sqrt{\frac{\log(2N_S/\delta)}{2M}}.
\]
Using the same fixed-selector complexity bound yields \eqref{eq:gen_unionS}.

\end{proof}

\subsection{Scoped risk decomposition and budget scaling}
\label{app:theory:risk_decomp}

The prediction-gap and generalization results can be combined at the risk level under an explicit
scope: the selector and tokenization are fixed or conditioned upon, and the theorem analyzes the
expensive relational learner rather than the entire end-to-end fine-tuning procedure.

For a fixed selector $S$, define the full-pool and budgeted risks for a relational map $f$ as
\[
\mathcal{R}_{\mathrm{full}}(f)
:=
\mathbb{E}_{(B,Y)}
\left[\ell\!\left(Y, f(Z(B))\right)\right],
\qquad
\mathcal{R}_{K}(f,S)
:=
\mathbb{E}_{(B,Y)}
\left[\ell\!\left(Y, f(X_S(B))\right)\right].
\]
Here $\mathcal{F}_K$ denotes the same parameterized relational maps evaluated either on the
conceptual full pool or on the masked budgeted input; the corollary applies to comparators for
which both evaluations are well defined.
Let $f_{\mathrm{full}}^\star\in\arg\min_{f\in\mathcal{F}_K}\mathcal{R}_{\mathrm{full}}(f)$ and
let $\widehat f_{K,S}$ be a $\xi_{\mathrm{opt}}$-approximate empirical risk minimizer for the
budgeted class:
\[
\widehat{\mathcal{R}}(\widehat f_{K,S},S)
\le
\inf_{f\in\mathcal{F}_K}\widehat{\mathcal{R}}(f,S)+\xi_{\mathrm{opt}}.
\]

\begin{corollary}[Risk-level budget decomposition under fixed selector]
\label{cor:risk_decomp}
Assume \Cref{assm:G0,assm:Ginput,assm:Gloss,assm:Gcap}. Also assume
Assumptions~\ref{assm:A0}--\ref{assm:A2} hold for the comparator
$f_{\mathrm{full}}^\star$ with the same $R$ and $L_\star$. For a fixed selector $S$, define
\[
\bar\varepsilon_K
:=
\mathbb{E}_{B}\!\left[\psi_K(B)+\Delta_K^{\mathrm{w}}(B)\right],
\qquad
\Gamma_K(\delta)
:=
2L_\ell C_{\mathrm{rel}}\frac{R\sqrt{K}}{\sqrt{M}}
+
3\sqrt{\frac{\log(2/\delta)}{2M}}.
\]
Then, with probability at least $1-\delta$,
\begin{equation}
\mathcal{R}_K(\widehat f_{K,S},S)
-
\mathcal{R}_{\mathrm{full}}(f_{\mathrm{full}}^\star)
\le
2R L_\ell L_\star\,\bar\varepsilon_K
+
2\Gamma_K(\delta)
+
\xi_{\mathrm{opt}}.
\label{eq:risk_decomp}
\end{equation}
\end{corollary}

\begin{proof}
By Lemma~\ref{lem:rad_gen}, contraction, and Assumption~\ref{assm:Gcap}, uniformly over
$f\in\mathcal{F}_K$,
\[
\left|\mathcal{R}_{K}(f,S)-\widehat{\mathcal{R}}(f,S)\right|
\le \Gamma_K(\delta).
\]
Therefore
\[
\mathcal{R}_K(\widehat f_{K,S},S)
\le
\widehat{\mathcal{R}}(\widehat f_{K,S},S)+\Gamma_K
\le
\widehat{\mathcal{R}}(f_{\mathrm{full}}^\star,S)+\xi_{\mathrm{opt}}+\Gamma_K
\le
\mathcal{R}_K(f_{\mathrm{full}}^\star,S)+\xi_{\mathrm{opt}}+2\Gamma_K.
\]
Since $\ell$ is $L_\ell$-Lipschitz, Theorem~\ref{thm:approx_formal} applied to
$f_{\mathrm{full}}^\star$ gives
\[
\mathcal{R}_K(f_{\mathrm{full}}^\star,S)
-
\mathcal{R}_{\mathrm{full}}(f_{\mathrm{full}}^\star)
\le
L_\ell\,\mathbb{E}_B
\left[
\left|\hat Y_{\mathrm{full}}(B)-\hat Y_K(B)\right|
\right]
\le
2R L_\ell L_\star\,\bar\varepsilon_K.
\]
Combining the two displays yields \eqref{eq:risk_decomp}.
\end{proof}

\begin{corollary}[Budget scaling under influence compressibility]
\label{cor:budget_scaling}
Suppose the averaged oracle tail and selector regret satisfy
\[
\mathbb{E}_{B}[\psi_K(B)]\le C_\psi K^{1-\alpha},
\qquad
\mathbb{E}_{B}[\Delta_K^{\mathrm{w}}(B)]\le C_\Delta K^{-\beta},
\]
for some $\alpha>1$ and $\beta>0$. Ignoring constants and lower-order concentration terms, the
right-hand side of \eqref{eq:risk_decomp} has the form
\begin{equation}
\mathcal{B}(K)
\lesssim
K^{1-\alpha}
+
K^{-\beta}
+
\sqrt{\frac{K}{M}}.
\label{eq:budget_scaling_bound}
\end{equation}
If selector regret is lower order than the oracle tail, the balancing budget scales as
\[
K^\star \asymp M^{\frac{1}{2\alpha-1}}.
\]
If selector regret is the dominant decreasing term, the corresponding balance is
\[
K^\star \asymp M^{\frac{1}{2\beta+1}}.
\]
\end{corollary}

\begin{proof}
The first claim substitutes the assumed rates into \eqref{eq:risk_decomp}. Let
$p=\alpha-1>0$. Balancing $K^{-p}$ with $K^{1/2}M^{-1/2}$ gives
$K^{p+1/2}\asymp M^{1/2}$, hence $K^\star\asymp M^{1/(2p+1)}
=M^{1/(2\alpha-1)}$. Replacing $p$ by $\beta$ gives the selector-regret-dominated scaling.
\end{proof}

\subsection{Practical Guidance for Budget Selection}
\label{app:theory:k_selection}

We now discuss how the theoretical bounds can inform the practical selection of the budget $K$.

\paragraph{Bias-variance tradeoff in $K$.}
Theorems~\ref{thm:approx_formal} and~\ref{thm:gen_formal} reveal a fundamental bias-variance tradeoff
controlled by $K$:
\begin{itemize}
  \item \textbf{Approximation error (bias):} By Theorem~\ref{thm:approx_formal},
  the gap between the full-information predictor and the budgeted predictor is bounded by
  $2R L_\star(\psi_K+\Delta_K^{\mathrm{w}})$.
  The oracle tail $\psi_K$ decreases with $K$, while the selector regret term measures how well the
  cheap selector tracks influential instances.
  \item \textbf{Generalization error (variance):} By Theorem~\ref{thm:gen_formal},
  the dominant generalization term scales as $2L_\ell C_{\mathrm{rel}}R\sqrt{K}/\sqrt{M}$,
  which \emph{increases} with $K$
  (larger input space leads to higher model capacity and potential overfitting).
\end{itemize}

\paragraph{Scoped combined bound.}
Under the fixed-selector scope of Corollary~\ref{cor:risk_decomp}, the budget-dependent part of
the bound is
\begin{equation}
\begin{split}
\mathcal{B}(K)
&\lesssim
\underbrace{2R L_\ell L_\star \cdot \bar\varepsilon_K}_{\text{budget approximation}}
\\&+
\underbrace{C_{\mathrm{gen}} \cdot \frac{R\sqrt{K}}{\sqrt{M}}}_{\text{relational estimation}}
\\&+ \text{lower-order terms}.
\end{split}
\label{eq:bias_variance_K}
\end{equation}

If the oracle influence tail decays as $\psi_K \le C K^{1-\alpha}$ for some $\alpha > 1$ and
selector regret is lower order, Corollary~\ref{cor:budget_scaling} gives the balancing rule
\begin{equation}
K^{1-\alpha} \sim \frac{\sqrt{K}}{\sqrt{M}}
\quad\Longrightarrow\quad
K^* \sim M^{\frac{1}{2\alpha - 1}}.
\label{eq:optimal_K}
\end{equation}

For example:
\begin{itemize}
  \item If $\alpha = 2$ (fast tail decay): $K^* \sim M^{1/3}$.
  With $M \approx 5000$ training pairs, this gives $K^* \approx 17$.
  \item If $\alpha = 3/2$ (moderate tail decay): $K^* \sim M^{1/2}$.
  With $M \approx 5000$, this gives $K^* \approx 71$.
  \item If $\alpha = 3$ (very fast tail decay): $K^* \sim M^{1/5}$.
  With $M \approx 5000$, this gives $K^* \approx 5.5$.
\end{itemize}

\begin{remark}[Consistency with empirical $K^*$]
Our empirical operating point $K^* = 64$ is consistent with a moderate tail decay regime
(e.g., $\alpha \approx 3/2$),
which aligns with the heavy-tailed CTS pool structure observed in miRNA targeting
(many low-influence candidates, a few high-influence ones; see Figure~\ref{fig:n_distribution}).
The performance saturation observed in Fig.~\ref{fig:perf_vs_k} around $K \ge 64$
is consistent with the approximation tail and selector regret becoming small at this budget.
\end{remark}

\paragraph{Practical recommendation.}
While the exact optimal $K^*$ depends on the unknown tail parameter $\alpha$,
the theory suggests the following practical guidelines:
\begin{enumerate}
  \item \textbf{Start with $K \approx \sqrt{M}$} as a reasonable default
  (this corresponds to the moderate tail decay regime).
  \item \textbf{Sweep $K$ on validation data} to find the empirical optimum,
  using the theory-predicted monotone improvement in approximation
  and potential overfitting at large $K$ as guidance.
  \item \textbf{Monitor the approximation-generalization tradeoff:}
  if increasing $K$ improves training loss but not validation loss,
  the generalization term is dominating and $K$ should be reduced.
\end{enumerate}

\subsection{Discussion of the $\sqrt{K}$ dependence}
\label{app:theory:tightness}

We briefly clarify what the $\sqrt{K}$ term means.

\begin{remark}[Upper-bound interpretation]
\label{rem:mean_pooling_comparison}
The $\sqrt{K}$ term in Theorem~\ref{thm:gen_formal} is an upper bound for a capacity-controlled
relational class. It should not be read as a lower bound for every permutation-invariant aggregator:
more restrictive classes such as fixed mean pooling can have smaller complexity because they average
tokens before prediction. The point for BR-MIL is that, after selection, the expensive relational
learner receives a $K$-token object with Frobenius radius at most $R\sqrt{K}$; the raw candidate
count $n$ does not enter this relational capacity term unless it is reintroduced through the selector.
\end{remark}

\section{Algorithms and Pseudocode}\label{app:algorithms}

\begin{algorithm}[!htbp]
\caption{Three-stage training for BR-MIL on miRNA--mRNA targeting}
\label{alg:training}
\begin{algorithmic}[1]
\STATE \textbf{Stage 1 (expensive CTS encoder).} Train $e_\theta$ on miRNA--CTS pairs with binary loss (Sec.~\ref{sec:losses}).
\STATE \textbf{Stage 2 (cheap CTS encoder).} Train $\tilde e_{\tilde\theta}$ by distilling from $e_\theta$ using Eq.~\eqref{eq:distill}.
\STATE \textbf{Stage 3 (aggregator + joint fine-tune).} For each miRNA--mRNA pair $(\mu,\nu)$:
\STATE \quad Extract CTS candidates via ESA scanning/filtering $\Rightarrow \{(x_i,u_i)\}_{i=1}^{n}$.
\STATE \quad Cheap scan: compute $(\tilde h_i,\tilde z_i)=\tilde e_{\tilde\theta}(x_i,u_i)$ for all $i\in[n]$.
\STATE \quad Select $S$ via \code{STSelector} on $\{(\tilde h_i,\tilde z_i,p_i)\}_{i=1}^{n}$ with budget $K=\min(\code{kmax},n)$.
\STATE \quad Expensive encode only selected: \((h_i,z_i)=e_\theta(x_i,u_i)\) for \(i\in S\).
\STATE \quad Tokenize: \(t_i=[h_i\Vert z_i\Vert s_i^{\mathrm{esa}}\Vert p_i]\); pad/mask to \code{kmax}.
\STATE \quad Aggregate: \(z_{\text{pair}}=f_\phi(\{t_i\}_{i\in S})\); compute \(L_{\text{pair}}=\mathrm{BinaryLoss}(z_{\text{pair}},Y)\).

\STATE \quad \textbf{Warmup:} freeze $\theta$, update only $\phi$; \textbf{Joint FT:} unfreeze $\theta$, update $(\theta,\phi)$ for a few epochs.
\end{algorithmic}
\end{algorithm}

\begin{algorithm}[!htbp]
\caption{Budgeted inference for one miRNA--mRNA pair}
\label{alg:inference}
\begin{algorithmic}[1]
\STATE \textbf{Input:} $(\mu,\nu)$, cheap encoder $\tilde e_{\tilde\theta}$, selector, expensive encoder $e_\theta$, aggregator $f_\phi$, \texttt{kmax}=64
\STATE Extract CTS candidates by ESA scan/filter: $\{(x_i,u_i)\}_{i=1}^{n}$ with $s_i^{\mathrm{esa}}\ge 6$.
\STATE Compute $(\tilde h_i,\tilde z_i)=\tilde e_{\tilde\theta}(x_i,u_i)$ for all $i\in[n]$.
\STATE Select $S$ with $|S|=K=\min(\code{kmax},n)$ via \code{STSelector} (Top-$K_1$ + diversity).
\STATE Compute $(h_i,z_i)=e_\theta(x_i,u_i)$ for $i\in S$.
\STATE Form tokens $t_i=[h_i\Vert z_i\Vert s_i^{\mathrm{esa}}\Vert p_i]$, pad/mask to \code{kmax}.
\STATE Output $z_{\text{pair}}=f_\phi(\{t_i\}_{i\in S})$ and $\hat Y=\sigma(z_{\text{pair}})$.
\end{algorithmic}
\end{algorithm}

\section{Additional Details for STSelector}
\label{app:selector}

This appendix provides full algorithmic details of STSelector, which is summarized in the main text as
\emph{Top-$K$ exploitation + position coverage + embedding deduplication}.

\paragraph{Inputs.}
Given a bag with cheap embeddings $\{\tilde h_i\}_{i=1}^n$, cheap logits $\{\tilde z_i\}_{i=1}^n$, and normalized transcript positions $\{p_i\}_{i=1}^n$, the selector outputs a subset $S\subseteq[n]$ with $|S|\le K$.

\paragraph{Step A (Top-$K_1$ exploitation).}
We select
\[
S_1 = \mathrm{TopK}(\tilde z_i, K_1),
\]
favoring highly confident candidates under the cheap encoder.

\paragraph{Step B (Position binning).}
We partition candidates into $B$ bins according to transcript position $p_i\in[0,1]$.
Within each bin, we keep a heap of size $m$ containing the top-$m$ candidates by $\tilde z_i$.
This produces a reduced pool $C$ of size at most $B\cdot m$.

\paragraph{Step C (Embedding deduplication).}
To reduce redundancy, we compute a lightweight SimHash key on $\tilde h_i$ (sign bits of selected dimensions).
Within each bin, at most $c$ candidates are kept per hash key.

\paragraph{Step D (Balanced quota allocation).}
For each bin $b$, we compute a weight
\[
w_b = \sum_{i\in \mathrm{Top-}t(b)} \exp(\tilde z_i/\tau_w),
\]
where $\mathrm{Top-}t(b)$ denotes the top-$t$ candidates in bin $b$.
We allocate a quota of candidates to each bin proportional to $w_b$, enforcing at least one per bin when possible.

\paragraph{Step E (Merge and fill).}
We output
\[
S = \mathrm{dedup}(S_1 \cup S_2),
\]
and fill remaining slots by descending $\tilde z_i$ until $|S|=K$.

\paragraph{Complexity.}
STSelector runs in $\mathcal{O}(n\log m)$ time due to per-bin heaps, and is implemented entirely on CPU for low-latency inference.

\bigskip

\section{Loss Functions and Distillation Details}
\label{app:loss}

\paragraph{Binary loss.}
Given logit $z$ and label $y\in\{0,1\}$, we apply label smoothing
\[
\tilde y = 
\begin{cases}
0.95 & y=1,\\
0.05 & y=0.
\end{cases}
\]
We use weighted BCE loss
\[
\ell_{\mathrm{BCE}}(z,\tilde y) = - w\big[\tilde y\log\sigma(z) + (1-\tilde y)\log(1-\sigma(z))\big].
\]

\paragraph{Focal mixture.}
We optionally combine BCE with focal reweighting:
\[
\ell_{\mathrm{focal}} = \alpha_t(1-p_t)^\gamma \ell_{\mathrm{BCE}},
\]
with $\gamma=1$, $\alpha=0.4$.
Final loss is
\[
L = \lambda_{\mathrm{bce}} L_{\mathrm{BCE}} + \lambda_{\mathrm{focal}} L_{\mathrm{focal}},
\quad
(\lambda_{\mathrm{bce}},\lambda_{\mathrm{focal}})=(0.01,1).
\]

\paragraph{Distillation.}
The cheap encoder is trained via

\begin{equation}
\mathcal{L}_{\mathrm{distill}}
= (1-\alpha)\mathcal{L}_{\mathrm{sup}}
+ \alpha\,\mathcal{L}_{\mathrm{KD}}
+ \beta_{\mathrm{feat}}\,\mathcal{L}_{\mathrm{feat}}
+ \beta_{\mathrm{rel}}\,\mathcal{L}_{\mathrm{rel}} .
\end{equation}

with temperature $T=2$ and cosine schedule $\alpha:0.8\to0.5$.
We use $(\beta_{\mathrm{feat}},\beta_{\mathrm{rel}})=(0.1,1)$.

\paragraph{Ablation of $\mathcal{L}_{\mathrm{rel}}$.}
To assess the contribution of relational distillation, we ablate it
($\beta_{\mathrm{rel}}=0$) and retrain both Stage~2 and Stage~3.
Downstream Stage~3 F1 changes by less than $0.001$ and PR-AUC by less than $0.002$,
indicating that cheap encoder quality at the current dataset scale
is dominated by the supervised and logit distillation terms.
We retain $\mathcal{L}_{\mathrm{rel}}$ for completeness of the distillation formulation,
but do not claim a measurable benefit at the current dataset scale.

\bigskip

\section{Additional runtime profiling details}
\label{app:runtime_extra}

\paragraph{Goal and setting.}
This appendix documents the profiling protocol behind \Cref{fig:compute_vs_k,fig:stage_breakdown}.
We measure \emph{online} inference cost for three pipelines:
\textsc{BR-MIL\_online} (cheap scan + STSelector + expensive encode on $K$ + Set Transformer),
\textsc{TargetNet\_like\_online} (window/CTS scoring + pooling; budget-independent),
and a heavier \textsc{Naive\_online} variant that performs more per-candidate computation before aggregation.

\paragraph{Hardware and software.}
Runtime profiling in \Cref{fig:compute_vs_k,fig:stage_breakdown} is conducted on a single NVIDIA RTX 4090 GPU (24\,GB), separate from large-scale MTI training on 8$\times$A100 GPUs.
Profiling runs use PyTorch~2.4.1 (\texttt{+cu121}) with CUDA~12.1 and cuDNN~9.1.0; CPU/thread settings and scripts are provided in the supplementary code/configs.

\paragraph{Measurement protocol.}
For each configuration (pipeline and $K$), we run a warmup phase followed by timed iterations.
We report wall-clock end-to-end latency and throughput, and record peak GPU memory allocation.
To reduce noise:
(i) GPU timings are synchronized (e.g., via \texttt{torch.cuda.synchronize()}) at measurement boundaries;
(ii) measurements use fixed batch sizes and identical input queues across pipelines where applicable;
(iii) we repeat runs and report mean$\pm$std over the same $R$ seeds used in the main experiments unless stated otherwise.

\paragraph{Definition of stages in \Cref{fig:stage_breakdown}.}
The stage breakdown aggregates profiled blocks that appear in all pipelines:
\begin{itemize}
  \setlength{\itemsep}{0.2em}
  \setlength{\topsep}{0.2em}
  \item \textbf{CTS generation / filtering:} ESA scan over the 3$^\prime$UTR and candidate filtering.
  \item \textbf{CPU gather (\texttt{gather\_all\_cpu}):} CPU-side packing/gathering of candidate tensors and metadata into contiguous buffers for subsequent model calls.
  \item \textbf{Scan / cheap forward (\texttt{scan\_or\_cheap\_forward}):} per-candidate evaluation of the cheap encoder (or the corresponding per-site scorer in non-BR-MIL baselines).
  \item \textbf{Selection (STSelector):} CPU-only subset selection producing indices $S$ with $|S|=K$.
  \item \textbf{Expensive encode on $K$:} forward pass of the expensive CTS encoder applied only to selected candidates.
  \item \textbf{Tokenize + pack:} token concatenation (Eq.~\eqref{eq:token}), padding/masking to \texttt{kmax}, and device transfers if needed.
  \item \textbf{Aggregation:} Set Transformer (SAB stack + PMA) producing the pair logit.
\end{itemize}
Stage names in the figures correspond to these blocks; minor implementation-level sub-stages are merged for readability.

\paragraph{Peak VRAM reporting.}
Peak VRAM is measured as the maximum allocated GPU memory during the forward pass under the same input batch and $K$.
Since \textsc{TargetNet\_like\_online} does not allocate token batches of size $K$ nor run attention over $K$ tokens,
its VRAM usage is largely flat across the $K$ sweep, whereas budgeted relational pipelines scale with $K$ due to token packing and attention.

\paragraph{Absolute numbers and configuration table.}
The released supplementary code/configs include the exact profiling scripts and configuration files used to reproduce the latency, throughput, and VRAM measurements.

\section{Split Sensitivity Analysis}
\label{app:split_sensitivity}

Our main evaluation uses a 10-fold balanced cross-validation protocol, where each of the 10 miRAWtest subsets serves as a held-out test fold once.
This design inherently provides robustness against partition artifacts, as every pair appears in exactly one test fold.
The per-fold F1 ranges from $0.812$ to $0.869$ (mean $0.840{\pm}0.022$) and PR-AUC ranges from $0.814$ to $0.877$ (mean $0.869{\pm}0.031$), confirming that performance is consistent across all 10 test partitions.

\section{Cross-Domain Validation Details}
\label{app:cross_domain}

\subsection{CAMELYON16 Experimental Setup}
\label{app:camelyon_setup}

\paragraph{Dataset.}
CAMELYON16 \cite{camelyon16} is a pathology MIL benchmark for metastasis detection in lymph node whole-slide images.
We use pre-extracted ResNet-50 Barlow Twins features (2048-dim) from the torchmil benchmark \cite{torchmil}.
The training split contains 270 slides (111 positive, 159 negative); the test split contains 129 slides (49 positive, 80 negative).
Bag sizes range from 140 to 44{,}402 patches per slide (mean 8{,}764).

\paragraph{Protocol.}
We follow the torchmil Table~1 protocol: 5-fold StratifiedKFold on 270 training slides (shuffle=True, random\_state=42).
Best model selected by validation accuracy.
Metrics reported as mean $\pm$ std over 5 folds.

\paragraph{Model architecture.}
\textbf{ABMIL baseline}: Gated Attention MIL with hidden\_dim=256, attn\_dim=128, dropout=0.25.
\textbf{BR-MIL}: Cheap scorer: Linear(2048$\to$512)$\to$ReLU$\to$Dropout$\to$Linear(512$\to$1).
Top-$K$ selection by cheap scores.
Set Transformer backbone: SAB$\times$2 (d\_model=256, n\_heads=8, d\_ff=1024) + PMA(k=1) + classifier.
Gradient checkpointing is applied to SAB layers.

\paragraph{Training.}
Adam (lr=1e-4, weight\_decay=1e-4), CosineAnnealingLR.
Gradient accumulation (8 steps), gradient clipping (max\_norm=1.0).
BCEWithLogitsLoss.
BR-MIL pre-trains the cheap scoring network for 12 epochs, freezes it, and then trains the Set Transformer backbone for 100 epochs.
ABMIL is trained for 50 epochs.

\subsection{Musk2 Experimental Setup}
\label{app:musk2_setup}

\paragraph{Dataset.}
Musk2 \cite{musk2} is a classic MIL benchmark for predicting musk vs.\ non-musk molecules.
It contains 102 bags with 166-dimensional chemical descriptor features.
Bag sizes are heavy-tailed: min=1, max=1044, mean=12.
No instance-level labels are available.

\paragraph{Protocol.}
10-fold StratifiedKFold cross-validation, 3 seeds $\{2020, 2025, 2026\}$, 200 epochs.
Metrics are reported as mean$\pm$std over 10 folds and 3 seeds.

\paragraph{Model architecture.}
\textbf{ABMIL}: Gated Attention with hidden\_dim=128, attn\_dim=64, dropout=0.25, 2-layer classifier.
\textbf{\BRMIL{}}: Cheap scorer: Linear(166$\to$128)$\to$ReLU$\to$Dropout$\to$Linear(128$\to$1).
We select the top $K{=}4$ instances by cheap score.
ISAB backbone (d\_model=128, n\_heads=4, n\_inds=16) + PMA(k=1) + classifier.
Auxiliary BCE loss ($\lambda{=}0.1$) on cheap scores to prevent gradient breakage.

\section{Full Architecture Ablations}
\label{app:arch_ablation}

This section provides the complete architecture sweeps referenced in \cref{tab:arch_ablation}.
All experiments use $K{=}64$ and are evaluated on the MTI validation set.

\subsection{Instance encoder scaling}

\Cref{tab:encoder_ablation} reports CTS-level validation F1 for four encoder configurations.
Performance saturates at ${\sim}909$K parameters (X-Large).

\begin{table}[h]
\centering
\caption{\textbf{Instance encoder scaling on MTI CTS-level validation.}
Channels, blocks, and multi-scale convolution (MS) control encoder capacity.}
\label{tab:encoder_ablation}
\small
\setlength{\tabcolsep}{4pt}
\renewcommand{\arraystretch}{1.05}
\begin{tabular}{lcccccc}
\toprule
Variant & Channels & Blocks & MS & Params & Emb dim & Val F1 \\
\midrule
Standard & [16,16,32,32] & [1,1,1,1] & & ${\sim}14$K & 384 & 0.6775 \\
Large & [32,32,64,64] & [2,2,2,2] & & ${\sim}153$K & 768 & 0.6838 \\
\textbf{X-Large} & [64,64,128,128] & [3,3,3,3] & \checkmark & $\mathbf{{\sim}909}$K & \textbf{1536} & \textbf{0.6849} \\
XX-Large & [128,128,256,256] & [4,4,4,4] & \checkmark & ${\sim}3.6$M & 3072 & 0.6840 \\
\bottomrule
\end{tabular}
\end{table}

\subsection{Aggregator family}

\Cref{tab:agg_family} compares three aggregator families under the same pipeline ($K{=}64$, X-Large encoder).

\begin{table}[h]
\centering
\caption{\textbf{Aggregator family comparison on MTI validation ($K{=}64$).}}
\label{tab:agg_family}
\small
\setlength{\tabcolsep}{4pt}
\renewcommand{\arraystretch}{1.05}
\begin{tabular}{lccc}
\toprule
Aggregator & Architecture & Params & Val F1 \\
\midrule
CNN & Sorted 1D-CNN, dilated & ${\sim}14$M & 0.6850 \\
GNN & $k$-NN graph + GAT, $L{=}3$, $k{=}8$ & ${\sim}37$M & 0.7602 \\
\textbf{SAB} & \textbf{Set Transformer, $L{=}4$, $H{=}16$} & $\mathbf{{\sim}18}$M & \textbf{0.7715} \\
\bottomrule
\end{tabular}
\end{table}

\subsection{Set Transformer capacity sweep}

\Cref{tab:agg_scaling} reports the full sweep over Set Transformer capacity.
All configurations use SAB attention unless noted otherwise.

\begin{table}[h]
\centering
\caption{\textbf{Full Set Transformer scaling sweep on MTI validation ($K{=}64$).}
$d$: model dimension; $L$: number of SAB layers; FF: feed-forward hidden dimension; $H$: number of attention heads.}
\label{tab:agg_scaling}
\small
\setlength{\tabcolsep}{4pt}
\renewcommand{\arraystretch}{1.05}
\begin{tabular}{llcccccc}
\toprule
Exp. & Attention & $d$ & $L$ & FF & $H$ & Drop & Val F1 \\
\midrule
A & SAB & 256  & 2 & 1024 & 8  & 0.1 & 0.7273 \\
B & SAB & 512  & 3 & 2048 & 8  & 0.1 & 0.7308 \\
C & SAB & 768  & 4 & 3072 & 12 & 0.1 & 0.7352 \\
\textbf{G} & \textbf{SAB} & \textbf{1024} & \textbf{4} & \textbf{4096} & \textbf{16} & \textbf{0.1} & \textbf{0.7353} \\
H & SAB & 1280 & 4 & 5120 & 16 & 0.1 & 0.7333 \\
F & SAB & 1024 & 5 & 4096 & 16 & 0.1 & 0.7335 \\
D & ISAB & 512 & 3 & 2048 & 8 & 0.1 & 0.6417 \\
\bottomrule
\end{tabular}
\end{table}

The SAB-based configurations show diminishing returns beyond $d{=}1024$, $L{=}4$ (Exp.~G).
The ISAB variant (Exp.~D) underperforms substantially ($0.6417$ vs.\ $0.7353$ for the comparable SAB Exp.~B at $d{=}512$), confirming that inducing-point approximation is too aggressive at the budget levels tested.

\newpage
\section*{NeurIPS Paper Checklist}

\begin{enumerate}

\item {\bf Claims}
    \item[] Question: Do the main claims made in the abstract and introduction accurately reflect the paper's contributions and scope?
    \item[] Answer: \answerYes{}
    \item[] Justification: The abstract and introduction state the main contributions: the \BRMIL{} formulation, the \PAIRFormer{} scan--select--aggregate architecture, theoretical results linking the budget $K$ to approximation and generalization, and empirical validation on miRNA benchmarks, MTI, and cross-domain MIL tasks. These claims are supported by the theoretical analysis (Section~\ref{sec:theory}) and experiments (Section~\ref{sec:experiments}).

\item {\bf Limitations}
    \item[] Question: Does the paper discuss the limitations of the work performed by the authors?
    \item[] Answer: \answerYes{}
    \item[] Justification: A limitations paragraph at the end of Section~\ref{sec:experiments} discusses reliance on candidate generation, constraints of available pair-level datasets, and the current single-miRNA modeling assumption.

\item {\bf Theory assumptions and proofs}
    \item[] Question: For each theoretical result, does the paper provide the full set of assumptions and a complete (and correct) proof?
    \item[] Answer: \answerYes{}
    \item[] Justification: The paper states two theoretical results, Theorems~\ref{thm:approx} and~\ref{thm:gen}. Their assumptions are referenced in the theorem statements, and complete assumptions, proofs, and practical guidance are provided in Appendix~\ref{app:theory}.

    \item {\bf Experimental result reproducibility}
    \item[] Question: Does the paper fully disclose all the information needed to reproduce the main experimental results of the paper to the extent that it affects the main claims and/or conclusions of the paper (regardless of whether the code and data are provided or not)?
    \item[] Answer: \answerYes{}
    \item[] Justification: The paper describes the model architecture and training/inference pipeline in Section~\ref{sec:method}, data splits and evaluation protocols in Section~\ref{sec:experiments}, and additional algorithms, losses, hyperparameters, architecture ablations, and profiling details in the appendices. An anonymized implementation with configuration files is included in the supplementary material, and the full codebase will be released upon acceptance.

\item {\bf Open access to data and code}
    \item[] Question: Does the paper provide open access to the data and code, with sufficient instructions to faithfully reproduce the main experimental results, as described in supplemental material?
    \item[] Answer: \answerYes{}
    \item[] Justification: An anonymized implementation is included in the supplementary material, including training/evaluation scripts, configuration files, and data preparation instructions for the reported experiments. Because the processed MTI artifacts are larger than the supplementary ZIP limit, we provide scripts and metadata to reconstruct the benchmark from public sources, and the full codebase, checkpoints, and processed MTI release package will be made public upon acceptance.

\item {\bf Experimental setting/details}
    \item[] Question: Does the paper specify all the training and test details (e.g., data splits, hyperparameters, how they were chosen, type of optimizer) necessary to understand the results?
    \item[] Answer: \answerYes{}
    \item[] Justification: Data splits, evaluation protocols, baselines, and compute settings are described in Section~\ref{sec:experiments}. The model architecture is described in Section~\ref{sec:method}, with additional algorithms, selector details, losses, hyperparameters, runtime profiling, cross-domain settings, and architecture sweeps provided in the appendices.

\item {\bf Experiment statistical significance}
    \item[] Question: Does the paper report error bars suitably and correctly defined or other appropriate information about the statistical significance of the experiments?
    \item[] Answer: \answerYes{}
    \item[] Justification: Main miRAW and deepTargetPro results report mean$\pm$std over three random seeds in Tables~\ref{tab:overall} and~\ref{tab:deeptargetpro}. Cross-domain experiments report mean$\pm$std over the corresponding folds and seeds. Large-scale MTI budget sweeps are reported as single large-scale runs due to compute cost, with compute details provided in Section~\ref{subsec:budget_analysis}.

\item {\bf Experiments compute resources}
    \item[] Question: For each experiment, does the paper provide sufficient information on the computer resources (type of compute workers, memory, time of execution) needed to reproduce the experiments?
    \item[] Answer: \answerYes{}
    \item[] Justification: The paper reports that small-scale experiments run on a single RTX 4090 GPU, large-scale MTI training uses 8$\times$A100 GPUs, and the $K=512$ MTI model requires approximately 65 minutes per epoch. Runtime and memory profiling details are provided in Appendix~\ref{app:runtime_extra}.
    
\item {\bf Code of ethics}
    \item[] Question: Does the research conducted in the paper conform, in every respect, with the NeurIPS Code of Ethics \url{https://neurips.cc/public/EthicsGuidelines}?
    \item[] Answer: \answerYes{}
    \item[] Justification: We have reviewed the NeurIPS Code of Ethics. The work uses publicly available biological datasets or derived non-identifying molecular interaction records, does not involve human subjects or personal data, and is intended for computational biology research.

\item {\bf Broader impacts}
    \item[] Question: Does the paper discuss both potential positive societal impacts and negative societal impacts of the work performed?
    \item[] Answer: \answerYes{}
    \item[] Justification: The work may positively support biological research and target discovery by improving functional miRNA--mRNA interaction prediction. Potential negative impacts are limited because the method operates on public, non-identifying molecular data and is not a clinical decision system; incorrect predictions could still mislead downstream biological hypotheses if used without experimental validation.
    
\item {\bf Safeguards}
    \item[] Question: Does the paper describe safeguards that have been put in place for responsible release of data or models that have a high risk for misuse (e.g., pre-trained language models, image generators, or scraped datasets)?
    \item[] Answer: \answerNA{}
    \item[] Justification: The released artifacts are prediction code, model checkpoints, and curated non-identifying biological interaction data. They are not generative models, surveillance systems, scraped media datasets, or other artifacts with high misuse risk requiring special safeguards.

\item {\bf Licenses for existing assets}
    \item[] Question: Are the creators or original owners of assets (e.g., code, data, models), used in the paper, properly credited and are the license and terms of use explicitly mentioned and properly respected?
    \item[] Answer: \answerYes{}
    \item[] Justification: Existing datasets, baselines, and data sources are cited in the paper, including TargetNet/miRAW, deepTargetPro, and the CLASH, chiRA, and HYBRID sources used to curate MTI. License and usage information for released derived assets will be documented in the supplementary release package.

\item {\bf New assets}
    \item[] Question: Are new assets introduced in the paper well documented and is the documentation provided alongside the assets?
    \item[] Answer: \answerYes{}
    \item[] Justification: The paper introduces the MTI benchmark, a curated dataset compiled from publicly available miRNA-target interaction sources. Its construction, statistics, candidate-pool distributions, and evaluation protocol are described in Appendix~\ref{app:dataset_stats} and the experimental section. The release package will include documentation, preprocessing scripts, and license/usage notes. The dataset contains molecular interaction records rather than human-subject data, so individual consent is not applicable.

\item {\bf Crowdsourcing and research with human subjects}
    \item[] Question: For crowdsourcing experiments and research with human subjects, does the paper include the full text of instructions given to participants and screenshots, if applicable, as well as details about compensation (if any)?
    \item[] Answer: \answerNA{}
    \item[] Justification: The paper does not involve crowdsourcing or research with human subjects.

\item {\bf Institutional review board (IRB) approvals or equivalent for research with human subjects}
    \item[] Question: Does the paper describe potential risks incurred by study participants, whether such risks were disclosed to the subjects, and whether Institutional Review Board (IRB) approvals (or an equivalent approval/review based on the requirements of your country or institution) were obtained?
    \item[] Answer: \answerNA{}
    \item[] Justification: The paper does not involve research with human subjects, so IRB approval or equivalent review is not applicable.

\item {\bf Declaration of LLM usage}
    \item[] Question: Does the paper describe the usage of LLMs if it is an important, original, or non-standard component of the core methods in this research? Note that if the LLM is used only for writing, editing, or formatting purposes and does \emph{not} impact the core methodology, scientific rigor, or originality of the research, declaration is not required.
    \item[] Answer: \answerNA{}
    \item[] Justification: The core method, experiments, and data analysis do not use LLMs as a component. Any language-model assistance was limited to writing, editing, or formatting and does not affect the scientific contribution.

\end{enumerate}

\end{document}